%% file: acl_latex.tex
\pdfoutput=1
    
\documentclass[11pt]{article}

\usepackage{acl}
\usepackage{custom}
\usepackage{bbm}

\usepackage{times}
\usepackage{latexsym}

\usepackage[T1]{fontenc}

\usepackage[utf8]{inputenc}

\usepackage{microtype}
\usepackage{inconsolata}

\usepackage{scalerel}
\usepackage{bbm}
\usepackage{xspace}

\usepackage{soul}

\usepackage{graphicx} 
\usepackage{subcaption}
\graphicspath{ {./images/} }

\usepackage{tikz}
\usetikzlibrary{bayesnet}

\usepackage{multirow}
\usepackage{comment}
\usepackage{booktabs}

\usepackage[inline]{enumitem}
\usepackage{enumitem} 

\usepackage{amsmath}
\usepackage{amssymb}
\usepackage{amsfonts}
\usepackage{amsthm}
\usepackage[thinc]{esdiff}
\usepackage{mleftright}
\input{macros}

\usepackage[noend]{algpseudocode}
\usepackage{algorithm}
\usepackage{cleveref}

\algrenewcommand\algorithmicindent{1.0em}%
\newcommand{\rightcomment}[1]{{ \(\textcolor{ETHGreen}{\smalltriangleright}\){\footnotesize\textit{ #1}}}}
\algrenewcommand{\algorithmiccomment}[1]{\hfill \rightcomment{#1}}  %
\algnewcommand{\LinesComment}[1]{\State {\color{black!50!green}\rightcomment{\parbox[t]{.95\linewidth-\leftmargin-\widthof{\(\smalltriangleright\) }}{#1}}}}
\algnewcommand{\LineComment}[1]{\State {\color{black!50!green}\smaller \(\smalltriangleright\) \parbox[t]{\linewidth-\leftmargin-\widthof{\(\smalltriangleright\) }}{\it #1}\smallskip}} %
\algnewcommand{\InlineComment}[1]{\hfill {\color{black!50!green}\(\smalltriangleright\) {\scriptsize \it #1}}}
\algrenewcommand\algorithmicindent{1.0em}%

\algrenewcommand\alglinenumber[1]{{\tiny\color{black!50}#1.}\hspace{-2pt}}
\newcommand{\algorithmicfunc}[1]{\textbf{def} {#1}:}
\algdef{SE}[FUNC]{Func}{EndFunc}[1]{\algorithmicfunc{#1}}{}
\makeatletter
\ifthenelse{\equal{\ALG@noend}{t}}%
{\algtext*{EndFunc}}
{}%
\makeatother

\usepackage{hyperref}
\usepackage[hyphenbreaks]{breakurl}

\usepackage[disable]{todonotes} %

\usepackage{emoji}
\newcommand{\ucopenhagen}{\emoji[emoji_src]{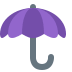}}
\newcommand{\ethz}{\emoji[emoji_src]{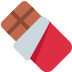}}
\newcommand{\nus}{\emoji[emoji_src]{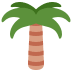}}

\title{A Probability--Quality Trade-off in Aligned Language Models and its Relation to Sampling Adaptors}

\author{
    Naaman Tan$^{\nus}$ \quad
    Josef Valvoda$^{\ucopenhagen}$\quad
    Tianyu Liu$^{\ethz}$\quad
    Anej Svete$^{\ethz}$\\
    \textbf{Yanxia Qin}$^{\nus}$\quad
    \textbf{Kan Min-Yen}$^{\nus}$\quad
    \textbf{Ryan Cotterell}$^{\ethz}$ \\
    $^{\nus}$National University of Singapore\quad\quad
    $^{\ucopenhagen}$University of Copenhagen\quad\quad
    $^{\ethz}$ETH Z{\"u}rich \\
    \{\burlalt{tannaaman}{mailto:tannaaman@u.nus.edu}, \burlalt{knmnyn}{mailto:knmnyn@nus.edu.sg}\}\texttt{@nus.edu.sg}\quad
    \burlalt{jval@di.ku.dk}{mailto:jval@di.ku.dk} \\ 
    \{\burlalt{tianyu.liu}{mailto:tianyu.liu@inf.ethz.ch}, \burlalt{asvete}{mailto:asvete@inf.ethz.ch},\burlalt{rcotterell}{mailto:rcotterell@inf.ethz.ch}\}\texttt{@inf.ethz.ch}
}

\begin{document}
\maketitle

\begin{abstract}
The relationship between the quality of a string, as judged by a human reader, and its probability, $\plm(\str)$ under a language model undergirds the development of better language models.
For example, many popular algorithms for sampling from a language model have been conceived with the goal of manipulating $\plm(\str)$ to place higher probability on strings that humans deem of high quality \citep{fan-etal-2018-hierarchical,holtzman2020curious}.
In this article, we examine the probability--quality relationship in language models explicitly aligned to human preferences, e.g., through reinforcement learning through human feedback. 
We show that, when sampling corpora from an aligned language model, there exists a trade-off between the strings' average reward and average log-likelihood under the prior language model, i.e., the same model before alignment with human preferences.
We provide a formal treatment of this phenomenon and demonstrate how a choice of sampling adaptor allows for a selection of how much likelihood we exchange for the reward.\looseness=-1

\vspace{0.5em}
\hspace{0.7em}\includegraphics[width=1.25em,height=1.25em]{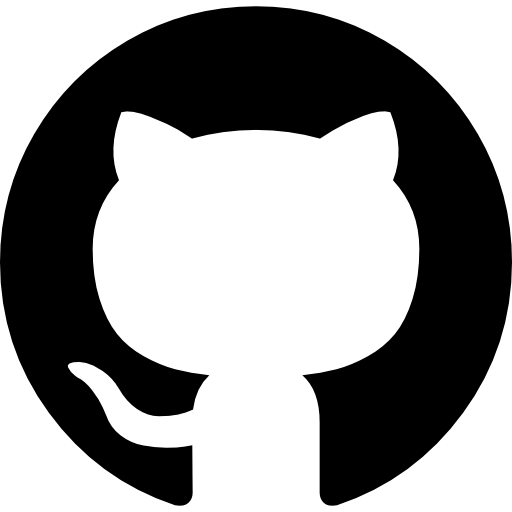}\hspace{0.50em}\parbox{\dimexpr\linewidth-7\fboxsep-2\fboxrule}{\url{https://github.com/tanyjnaaman/probability-quality-paradox}}
\vspace{0.2em}
\end{abstract}

\section{Introduction}
The relationship between the probability of a string and its quality as judged by a human reader is fundamental to the goal of producing high-quality text from probabilistic language models. 
For example, the belief that a string $\str$ with a higher probability under a language model $\plm$ should be of higher quality \citep{graves2012sequence,fan-etal-2018-hierarchical,holtzman2020curious,zhang-etal-2021-trading} has motivated the development of sampling adaptors like top-$\topk$ \citep{fan-etal-2018-hierarchical} and nucleus sampling \citep{holtzman2020curious}.
The intuition behind this belief is that, under a well-calibrated language model trained primarily on human-written text, strings that occur with high probability should be more human-like, and thus judged by humans to be of higher quality.
Indeed, sampling methods that skew a language model towards high-probability strings have been shown to dramatically improve the quality of text sampled from a model \citep{wiher-etal-2022-decoding}.
Over the years, several studies have since contributed to a more nuanced understanding of the probability--quality relationship \citep{holtzman2020curious,zhang-etal-2021-trading,basu2021mirostat} and led to the development of more sophisticated sampling methods \citep{hewitt-etal-2022-truncation,meister-etal-2022-high}.\looseness=-1

This paper seeks to explain the relationship between probability and quality in the specific case of \defn{aligned} language models, i.e., language models explicitly fine-tuned to better match human preferences, e.g., through reinforcement learning with human feedback \citep[RLHF;][]{christiano2017deep,leike2018scalable,ziegler2020finetuning,stiennon2020rlhf,ouyang2022instruct,ouyang2022rlhf,korbak-etal-2022-rl,korbak2023rlhf}.
In particular, we provide a formal argument and empirical evidence that in RLHF-tuned models, there is an inherent \emph{anti-}correlation, i.e., a \emph{trade-off}, between probability and quality.

In the theoretical portion of this paper, we formalize the probability--quality trade-off in aligned language models.
Specifically, we show that for corpora of generated strings of a large enough size, the average log-probability under the \defn{prior} (unaligned) language model $\plm$ trades off with the average score assigned by a reward model \citep{gao2023scalingreward}.
We refer to this phenomenon as a probability--quality trade-off because, in the context of aligned language models, quality is often operationalized by the reward model \citep{perez-etal-2022-red,lee2024aligning}.
This trade-off follows straightforwardly from 
an application of a concentration inequality and can be seen as a direct corollary of the asymptotic equipartition property \citep[AEP;][]{elements-of-info-theory}.
In addition to the standard version of the AEP based on Chebyshev's inequality, we also prove a tighter version of the AEP based on a Chernoff bound that applies to certain language models, including Transformer-based models \citep{vaswani2017attention}.
Lastly, we show how an appropriate choice of a globally normalized sampling adaptor can control the probability of a generated string under the prior $\plm$.
This finding implies that they can be used to choose a point on the probability--quality trade-off and control the average quality of the generated text.
Interestingly, the trade-off predicts the emergence of Simpson's paradox, which we also observe in our experiments.\looseness=-1

In the empirical part of our paper, we present two sets of experiments.
First, with synthetic data, we construct toy language and reward models to further demonstrate our theoretical results.
This gives easily reproducible empirical evidence for our claim.
Then, with a second set of experiments, we show that this trade-off exists in practice with open-sourced RLHF-tuned models and that globally normalized sampling adaptors do allow us to control where a corpus of generated text will lie on the trade-off.
Notably, we find that in practice, \defn{locally normalized sampling adaptors} \citep[e.g., nucleus sampling;][]{holtzman2020curious,meister-etal-2023-efficacy} are sufficiently close approximations of their globally normalized counterparts, and can be effectively used to control the trade-off. 

\section{The Probability--Quality Relationship}
Let $\alphabet$ be an \defn{alphabet}.
A \defn{string} $\str \in \kleene{\alphabet}$ is a finite list of symbols drawn from $\alphabet$.
The set of all strings is called $\alphabet$'s \defn{Kleene closure} and is denoted as $\kleene{\alphabet}$.
A \defn{language model} $\plm$ is a probability distribution over $\kleene{\alphabet}$.
It is often assumed that there is a \emph{positive} correlation between the quality of a string and its probability under a language model trained on human text $\plm(\str)$ \citep{graves2012sequence,fan-etal-2018-hierarchical,holtzman2020curious,zhang-etal-2021-trading}.\footnote{Probability here refers to string likelihood under a general-purpose, \emph{unconditional} language model. The distinction is important since we will examine the relationship between this probability and the quality of strings in an \emph{aligned} language model.\looseness=-1}
The idea behind this assumption is that under such a model, a higher probability string should be more human-like, and thus judged by humans to be of higher quality.
Because of this positive correlation, string probability is commonly used to reason about text quality in the context of natural language processing.
For example, a number of studies on language modeling methods report measures of perplexity to quantify the quality of text produced by the model \citep{vaswani2017attention,devlin2019bert,brown2020,mauveJMLR}.\looseness=-1

More recently, several authors have observed evidence that this positive correlation breaks down at extremes, finding instead that a string's probability is positively correlated with its quality up to an inflection point, after which it becomes negatively correlated \citep{yang-etal-2018-breaking,stahlberg-byrne-2019-nmt,holtzman2020curious,zhang-etal-2021-trading,meister-etal-2022-high}.
\citet{meister-etal-2022-high} show that this inflection point lies near the entropy of human-written text, and hypothesize that a string should encode a similar amount of information to natural language strings to be considered high-quality. 
This finding has since inspired various sampling adaptors, e.g., locally typical \citep{meister-etal-2023-locally} and $\eta$-sampling \citep{hewitt-etal-2022-truncation}.

In this paper, we investigate the probability--quality relationship in aligned language models.
There are two reasons for this choice.
First, aligned models have an additional constraint---they are fine-tuned to \emph{only} produce high-quality text (as judged by a reward model).
And, \textit{a priori}, it is unclear how alignment might influence the probability--quality relationship.
Second, aligned language models are mathematically similar to the conditional language models often found in machine translation and controlled text generation, for which prior work has found relationships not seen in unconditional language models \citep{callison-burch-etal-2007-meta,Banchs2015AdequacyFluencyME,teich2020handbook,sulem-etal-2020-semantic,lim2024simpsons}.\looseness=-1

\section{Learning from Human Feedback}\label{sec:rlhf}
Because our investigation focuses on aligned language models, we now introduce RLHF, a popular algorithm for alignment.
Given a task of interest, we are interested in producing strings that are aligned with human preferences for that task.
Let $\Aset = \setof{\bad, \good}$ denote binary judgments of alignment, and $\rvA$ be an $\Aset$-valued random variable.
Then, we assume the existence of a true human-aligned distribution over strings $\palign(\str) \defeq \plm(\str \mid \rvA = \good)$, such that strings with high probability under $\palign$ also receive positive scores from human annotators.
For example, we may desire that strings that are offensive should have a lower probability under $\palign$ for chat-related tasks.
With this, we can formalize the goal of alignment as obtaining an aligned language model $\qalign(\str)$ such that $\qalign(\str)$ is a good approximation to $\palign(\str)$.\looseness=-1

RLHF is a widely used method of finding such a model.
At the core of RLHF is a \defn{reward function}, which models preferences of human annotators, denoted as $\reward \colon \kleene{\alphabet} \rightarrow (-\infty, \rewardbound]$ where $\rewardbound$ is a bound in $\in \Rpos$.\footnote{We assume that the reward function is bounded, following \citet{levine2018reinforcement,korbak-etal-2022-rl}.\looseness=-1}
In practice, the reward function is itself typically parameterized by a neural network and derived by modeling preferences with a Bradley--Terry model \citep{bradley1952} and a ranked dataset.
The human-aligned language model $\palign(\str)$, reward function $\reward(\str)$ and prior language model $\plm(\str)$ can be related as follows \citep{korbak-etal-2022-rl}:\looseness=-1
\begin{equation}\label{eq:rlhf-bayes}
        \palign(\str) = \frac{\plm(\str) \exp \left(\frac{1}{\beta}\reward(\str)\right)}{\normalizingconstantrlhf},
\end{equation}
where $\beta \in \Rpos$ is a scaling factor and\looseness=-1
\begin{equation}\label{eq:normalizer-rlhf}
    \normalizingconstantrlhf\defeq  \sumoverstrings \plm(\str) \exp \left(\frac{1}{\beta}\reward(\str)\right).
\end{equation}
is the normalizing constant. 

If we take a variational inference perspective of RLHF \citep{levine2018reinforcement,korbak-etal-2022-rl}, then the goal of RLHF is to find an aligned language model $\qalign(\str)$ that minimizes the backward Kullback--Leibler (KL) divergence between $\qalign$ and the ground truth aligned distribution over strings $\palign$:\looseness=-1
\begin{subequations}\label{eq:rlhf-objective}
    \begin{align}
        &\KLdiv{\qalign}{\palign} \\
            &= \sumoverstrings \qalign(\str)  \log \frac{\qalign(\str)}{\palign(\str)}\\
            &= \sumoverstrings \qalign(\str) \log \frac{\qalign(\str) \normalizingconstantrlhf}{\plm(\str) \exp \left(\frac{1}{\beta}\reward(\str)\right)}\\
            &= \log \!\normalizingconstantrlhf + \KLdiv{\qalign}{\plm} -\! \frac{1}{\beta} \expected_{\str \sim \qalign} [\reward(\str)].
    \end{align}
\end{subequations}
This objective can be seen as KL-regularized reinforcement learning \citep{stiennon2020rlhf}.\looseness=-1

Notably, \emph{any} preference-aligned language model can be expressed in the framework of RLHF, even if no explicit reward function was used in training the model. 
As shown by \citet{rafailov2023direct}, for \emph{any} aligned language model $\qalign$ that minimizes the backward-KL objective, we can always construct a ``secret'' reward function $\reward_\qalign$ with:
\begin{equation}\label{eq:secret}
    \reward_\qalign(\str) = \beta \left( \log \frac{\qalign(\str)}{\plm(\str)} + \log \normalizingconstantrlhf \right).
\end{equation}
The implication of this is that the result we prove for RLHF-tuned models can be trivially extended to \emph{any} conditionally aligned language model, like the ones often deployed in controlled text generation \citep{hu2017toward,krause2020gedi,fudge2021,liu2021dexperts,zhang2023survey}.\looseness=-1

\section{Sampling Adaptors}\label{sec:sampling-adaptors}
We are often interested in sampling strings from a language model $\plm(\str)$.
Sampling is usually performed autoregressively, i.e., we sample a symbol $\eostoken \in \eosalphabet \defeq \alphabet \cup \setof{\eos}$ iteratively from $\plm(\cdot \mid \strltt)$ at each time step $t$ until the special \underline{e}nd-\underline{o}f-\underline{s}equence $\eos$ symbol is reached.
Locally normalized sampling adaptors \citep{meister-etal-2023-efficacy} are post-hoc alterations of $\plm(\cdot \mid \strltt)$ that have been shown to dramatically improve the quality of text produced by language models, and are often considered an integral part of a text generation pipeline \citep{wiher-etal-2022-decoding}.
Common examples include top-$\topk$ \citep{fan-etal-2018-hierarchical} and nucleus sampling \citep{holtzman2020curious}.\looseness=-1

Applying a locally normalized sampling adaptor to $\plm(\cdot \mid \strltt)$ at each $t$ is commonly formulated as the composition of two steps. 
First, the application of a \defn{transform function} $\globalsamplingadaptor \colon \simplextoken \rightarrow \Rpostoken$ that maps the distribution over $\eosalphabet$ to a vector of non-negative values.
The transform function is responsible for the core logic of the sampling adaptor, e.g., assigning symbols outside the top-$\topk$ zero probability. 
Second, a normalization step is performed to ensure a valid distribution over $\eosalphabet$, i.e, one where $\sumovereostokens \globalsamplingadaptor(\plm(\eostoken \mid \strltt)) = 1$. 
When a locally normalized sampling adaptor is applied to a language model $\plm$, we can express the induced distribution $\plmlocalsampling$ as\looseness=-1
\begin{align}\label{eq:local-sampling-adaptor}
    \plmlocalsampling(&\str) = \\
     &\frac{\globalsamplingadaptor(\plm(\eos \mid \str))}{\sumovereostokens{\globalsamplingadaptor(\plm(\eostoken \mid \str))}} \frac{\prodstring \globalsamplingadaptor(\plm(\tokent \mid \strltt))}{\prodstring \sumovereostokens \globalsamplingadaptor(\plm(\eostoken \mid \strltt))}\nonumber
\end{align}

Despite their empirical success, locally normalized sampling adaptors have several undesirable theoretical properties. 
First, they can induce a language model that is not \defn{tight}, i.e., one where probability mass can be placed on infinite-length strings \citep{welleck-etal-2020-consistency,du-etal-2023-measure}.
For instance, this can occur if zero probability is placed on $\eos$ at every step $t$. 
Second, the normalization operation at each $t$ can produce unexpected behavior. 
Specifically, because the denominator used to normalize the local distribution at each $t$ is dependent on the probability mass assigned to \emph{other} symbols, a string assigned a higher score with the transform function is not necessarily assigned a higher probability under the induced language model $\plmlocalsampling$.
This can lead to less desirable strings being sampled more frequently. 
See \Cref{app:top-k-example} for an example.

We thus introduce \defn{globally normalized sampling adaptors}, a procedure to sample from a language model that circumvents these issues.
Globally normalized sampling adaptors are also defined with respect to a transform function $\globalsamplingadaptor$, and thus every existing locally normalized sampling adaptor can be mapped to its globally normalized counterpart.
The transform function is similarly applied at each time step to $\plm(\cdot \mid \strltt)$.
But without the normalization step, we define the probability of a string under the induced model $\plmsampling$ as\looseness=-1
\begin{equation}\label{eq:samplingadaptor-stringprob}
    \plmsampling(\str) = \frac{\globalsamplingadaptor\big(\plm(\cdot \mid \str)\big)(\eos) \prodstring \globalsamplingadaptor\big(\plm(\cdot \mid \strltt)\big)(\tokent)}{\normalizingconstantstring}.
\end{equation}
where the normalizing constant $\normalizingconstantstring$ is defined\looseness=-1
\begin{align}\label{eq:samplingadaptor-stringprob-normalizingconstant}
    \normalizingconstantstring &=\\
    & \sumoverstrings \globalsamplingadaptor\big(\plm(\cdot \mid \str)\big)(\eos) \prodstring \globalsamplingadaptor\big(\plm(\cdot \mid \strltt)\big)(\tokent). \nonumber
\end{align}

The language model induced by a globally normalized sampling adaptor is always tight. 
Under the assumption that $\normalizingconstantstring$ is finite \citep{cotterell2023formal}, it is easy to see that $\sumoverstrings \plmsampling(\str) = 1$.
It is also easy to see that 
\begin{equation}
\plmsampling(\str) \propto \globalsamplingadaptor\big(\plm(\cdot \mid \str)\big)(\eos) \prodstring \globalsamplingadaptor\big(\plm(\cdot \mid \strltt)\big)(\tokent).
\end{equation}
That is, the probability of a string under the adapted model is determined only by the transform function without additional renormalization, which makes it convenient to reason about how a choice of transform function can modify the properties of the induced distribution $\plmsampling$.

\naamanchange{We show in \Cref{sec:trade-off-and-adaptors} how globally normalized sampling adaptors can be used to control the probability--quality trade-off.} 
To generate strings from a language model using a globally normalized sampling adaptor, we use the Independent Metropolis--Hastings Algorithm \citep[IMHA;][]{metropolis1953,hastings1970,wang2022metropolis}.\footnote{
Directly sampling from $\plmsampling$ is not possible because it requires computing the constant $\normalizingconstantstring$. 
This sum over the countably infinite set of finite strings is in practice unknown. 
Using IMHA avoids this issue as it only requires the knowledge of the target distribution up to a \emph{multiplicative constant}.\looseness=-1}
The IMHA is a Markov Chain Monte Carlo (MCMC) method that simulates sampling from a target distribution using a proposal distribution.
We detail how the IMHA's acceptance--rejection protocol can be used to generate text with a globally normalized sampling adaptor in \Cref{alg:imh}.
The idea is that by sampling sequentially from the proposal (i.e., the distribution induced by a locally normalized sampling adaptor, $\plmlocalsampling$) and appropriately accepting or rejecting samples, the generated Markov chain (i.e., the sequence of samples) converges to a stationary distribution equal to the target distribution (i.e., the one induced by a globally normalized sampling adaptor, $\plmsampling$).
Convergence is often diagnosed by the absence of autocorrelation between consecutive samples \citep{deonovic2017convergence}.
In the case of categorical variables like strings, this can be measured with Cram{\'e}r's V \citep{ialongo2016}.\looseness=-1

\begin{algorithm}[t]
    \begin{algorithmic}[1]
    \Func{\textsc{IMHA} $\left(N, \plmsampling, \plmlocalsampling\right)$}
        \State $\str \sim \plmlocalsampling$
        \State $\strset \gets \{\}$
        \For{$n \gets 1 \ldots N$}
             \State $\strprime \sim \plmlocalsampling$
             \State $\acceptancer(\str, \strprime) \gets \frac{\plmsampling(\strprime) \, \plmlocalsampling(\str)}{\plmsampling(\str) \, \plmlocalsampling(\strprime)} $
             \State $r \sim U(0,1)$ 
            \If{$\acceptancer > r$}
                \State $\strset \gets \strset \cup \{\strprime\}$ 
                \State $\str \gets \str'$
            \Else{}
                \State $\strset \gets \strset \cup \{\str\}$ 
            \EndIf
        \EndFor
        \State \Return $\strset$
        \EndFunc
    \end{algorithmic}
    \caption{IMH Algorithm.}
\label{alg:imh}
\end{algorithm}
\setlength{\textfloatsep}{4pt}

\section{Theoretical Results}\label{sec:bayesian}
We are now ready to discuss the theoretical contributions of this paper. 
In \Cref{sec:fundamental-trade-off} we begin with a formal argument that there exists a fundamental trade-off between the average log-probability under the prior and the average reward for corpora sampled from an aligned language model. 
Then, in \Cref{sec:trade-off-and-adaptors}, we show how sampling adaptors, by shifting probability mass, can be used to choose any point on this trade-off.
In \Cref{sec:simpsons}, we conclude the section with a discussion of how this trade-off leads to an emergence of Simpson's paradox.\looseness=-1 

\subsection{A Fundamental Trade-off}\label{sec:fundamental-trade-off}
Let $\qalign(\str)$ be an aligned language model such that $\qalign(\str) = \plm(\str \mid \rvA = \good)$.\footnote{This assumption is not strictly necessary, but allows us to discuss the trade-off in terms of the true reward function $\reward$, rather than $\qalign(\str)$'s ``secret'' reward function $\reward_{\qalign}$.}
Also, making use of a $\kleene{\alphabet}$-valued random variable $\rvY$ distributed according to $\qalign$, we define the \defn{pointwise joint entropy} of $\qalign(\str)$ as follows:\looseness=-1
\begin{equation}
    \ent(\rvY, \rvA = \good) \defeq -\sumoverstrings \qalign(\str) \log \qalign(\str).
\end{equation}
Now, we can introduce the $(N, \varepsilon)$-\defn{typical set} of $\qalign$:
\begin{equation}\label{eq:typical-set}
\begin{aligned}
   \typicalset(&\qalign) \defeq \Big\{ \strset \in (\kleene{\alphabet})^N \mid  \\ &  
    \Big|  \ent(\rvY, \rvA = \good) + \frac{\log \qalign(\strset)}{N} \Big| < \varepsilon \Big\},
\end{aligned}
\end{equation}
where $\strset$ is a \defn{corpus} of strings and $\log \qalign(\strset) = \sum_{n=1}^N \log \qalign(\strn)$ for some $\varepsilon > 0$.
In words, $\typicalset(\qalign)$ is the set of corpora of size $N$, i.e., bags of strings, each sampled from $\qalign$ with average information content $\frac{\log \qalign(\strset)}{N}$ close to the pointwise joint entropy $\ent(\rvY, \rvA = \good)$.\looseness=-1

This notion of typicality is useful because it can be shown that a sampled corpus $\strset = \setof{\strn}_{n=1}^N$ where $\strn \sim \qalign$ falls in $\typicalset(\qalign)$ with high probability. 
Let $\rvI = -\log \qalign(\rvY)$ be a random variable that denotes the information content of a string, and let $\var(\rvI)$ denote its variance.\footnote{$\var(\rvI)$, in the few papers that treat it directly, 
is often called the \defn{varentropy} \citep{varentropy}.\looseness=-1}
Then, with Chebyshev's inequality we can show that:\looseness=-1 
\begin{equation}\label{eq:chebyschev}
     \prob\mleft( \strset \notin  \typicalset(\qalign)\mright) \leq \frac{\var(\rvI)}{N \varepsilon^2} = \bigo \mleft(\frac{1}{N}\mright).
\end{equation}
The full derivation can be found in \Cref{app:proof-chebyshev}.
What \Cref{eq:chebyschev} says is that if we observe a set of $N$ strings, the probability that the corpus lies outside $\typicalset(\qalign) \rightarrow 0$ as $N\rightarrow \infty$.
Equivalently, this is to say that the \defn{sample entropy} $-\frac{\log \qalign(\strset)}{N}$ collapses around the entropy $\ent(\rvY, \rvA = \good)$ with high probability when $N$ is large. 

The above derivation is standard.
However, what is less standard is the application of Bayes' rule to show that
strings in the typical set display a fundamental trade-off.

\begin{restatable}[Probability--quality trade-off]{proposition}{tradeoff} \label{prop:trade-off}
\begin{equation}\label{eq:proof-trade-off-typical-set}
    \prob\left(\Big|C + \frac{\log \plm(\strset)}{N} + \frac{\reward(\strset)}{\beta N} \Big| < \varepsilon\right) > 1 - \delta 
\end{equation}
where $\delta = \bigo(\frac{1}{N})$ and $C \defeq \ent(\rvY \mid \rvA = \good) - \log \normalizingconstantrlhf$ is a constant, and we use the shorthands $\log \plm(\strset) = \sum_{n=1}^N \log \plm(\strn)$ and $\reward(\strset) = \sum_{n=1}^N \reward(\strn)$.\looseness=-1
\end{restatable}
\noindent
\Cref{prop:trade-off} says that a corpus $\strset$ of size $N$ sampled from $\qalign$ will have its average log-probability  $\frac{\log \plm(\strset)}{N}$ and average reward $\frac{\reward(\strset)}{\beta N}$ bound by a constant \emph{with high probability}. 
The implication of this is that the two quantities will trade off \emph{linearly}.

\begin{proofsketch}
\Cref{prop:trade-off} follows relatively straightforwardly from \Cref{eq:typical-set} and \Cref{eq:chebyschev}, when one observes from \Cref{eq:secret} that $\qalign(\str) \propto \plm(\str) \exp\mleft( \frac{1}{\beta} \reward(\str)\mright)$, which implies $\log(\qalign) = \log \plm(\str) + \frac{1}{\beta}\reward(\str) + \mathrm{constant}$.
Recall that corpora in the typical set $\typicalset(\qalign)$ have average information content $\frac{\log \qalign(\str)}{N}$ close to the \emph{constant} pointwise joint entropy $\ent(\rvY, \rvA = \good)$ (\Cref{eq:typical-set}). 
That is, typical corpora \emph{by definition} exhibit a trade-off between average log-probability under the prior $\frac{\log \plm(\str)}{N}$ and average reward $\frac{\reward(\str)}{N}$. 
Then, due to Chebyshev's inequality in \Cref{eq:chebyschev}, we have $\Big| C + \frac{\log \plm(\strset)}{N}  + \frac{\reward(\strset)}{\beta N} \Big| < \varepsilon$ with probability at least $\mleft( 1 - \frac{\var(\rvI)}{N \varepsilon^2} \mright)$ for all $N$ and $ \varepsilon > 0$.
When $N \ge \frac{\var(\rvI)}{ \delta \varepsilon^2}$ for some $\delta = \bigo(\frac{1}{N})$, the above holds  with probability at least $1 - \delta$, i.e., \emph{with high probability}, and we arrive at the proposition.
The full proof is provided in \Cref{app:proof-trade-off}.
\end{proofsketch}

Strictly speaking, \Cref{prop:trade-off} describes a trade-off between the average log prior probability and the average reward.
However, because the reward function is often used to reflect human preferences for various notions of quality, e.g., helpfulness or concision \citep{perez-etal-2022-red,ethayarajh2022}, we can interpret this result as a probability--quality trade-off.\looseness=-1

\paragraph{Assumptions.}
\Cref{prop:trade-off} relies on two key assumptions. 
First, that $\qalign$ has finite entropy.\footnote{
In general, this is not true.
See \Cref{sec:inf-entropy-lms} for an example.\looseness=-1
} 
Second, that the variance of information content of a string is also finite, i.e., $\var(\rvI) < +\infty$.
We argue that neither of these assumptions are limiting in practice because we show in \Cref{prop:boundedrenyi} and \Cref{prop:varrvibounded} that they hold for all Transformer-based language models, which constitute the base architecture for most modern models \citep{brown2020,touvron2023llama}.\looseness=-1

\paragraph{A Tighter Bound.}
We remark that there exists a tighter bound for the probability--quality trade-off than the $\bigo(\frac{1}{N})$ one in \Cref{eq:chebyschev} for specific types of language models. 
Specifically, we show in \Cref{sec:chernoff} that for transformer and $n$-gram based language models, the probability that sampled corpora land in the typical set $\typicalset(\qalign)$ and exhibit the trade-off grows \emph{exponentially} quickly, i.e., the bound is $\bigo\mleft(\exp(-cN)\mright)$ for some constant $c \in \Rpos$.

\subsection{Controlling the Trade-off}\label{sec:trade-off-and-adaptors}
Ideally, we would like to choose how much probability we trade for quality when sampling corpora from an aligned language model.
After all, depending on the context, it may be desirable to extract higher-reward text (e.g., to improve alignment) or lower-reward text \citep[e.g., to combat overfitting of the reward function;][]{azar2023general,gao2023scalingreward,wang2024secrets,he2024improving}.

\anejchange{We can leverage sampling adaptors (\Cref{sec:sampling-adaptors}) to exercise this control.
The intuition for this comes from analyzing the effect of a specific sampling adaptor on the prior string probability.
That is, in \cref{app:global-sampling-adaptors}, we show for particular sampling adaptor that
\begin{subequations}
\begin{align}\label{eq:sampling-adaptor-decomp-brief}
    \palignsampling(\str) &\propto \fprior \plm(\str) \exp\left(\frac{1}{\beta}\reward(\str)\right) \\ 
    & = \plmsampling(\str) \exp\left(\frac{1}{\beta}\reward(\str)\right),
\end{align}
\end{subequations}
where $\palignsampling$ is the induced distribution when applying a sampling adaptor to an aligned model $\palign$ and $\fprior$ refers to the aggregated series of truncation functions coming from the transform function $\globalsamplingadaptor$ that rely only on the \emph{prior} $\plm(\str)$.
\Cref{eq:sampling-adaptor-decomp-brief} says that in some special cases, the effect of applying a globally normalized sampling adaptor to $\palign$ is akin to applying it to the prior language model $\plm$ and then multiplying the result by the likelihood that the generated string aligns with human preferences, i.e., the effects of the sampling adaptor can be pushed to the prior.
}

\anejchange{Given the probability--quality trade-off, this suggests that an appropriate choice of sampling adaptor can be used to control the average log-probability of sampled corpora, which then determines the average reward of generated text.
For example, we could use the temperature sampling with a high temperature to produce lower probability (and thus higher reward) strings.
This intuition proves to be useful and leads to an efficient way to control the trade-off, as we demonstrate in \Cref{sec:langmodelexperiments}.}

\subsection{The Emergence of Simpson's Paradox}\label{sec:simpsons}
Following \citet{lim2024simpsons}, we now argue that the trade-off described in \Cref{prop:trade-off}---under appropriate conditions---can lead to the emergence of Simpson's paradox. 
Specifically, the paradox emerges when the reward $\reward(\str)$ is \textit{a priori} positively correlated with string likelihood under the prior language model $\plm(\str)$.
This is not always the case, of course.
However, we should expect it to be true when reward scores reflect the quality of text and the language model is well-calibrated.
Thus, if we consider samples $\str \sim \qalign$, we should expect $\log \plm(\str)$ to be positively correlated with $\reward(\str)$ \emph{by assumption}.
This correlation exists at the string level.\looseness=-1

Simultaneously, in \Cref{prop:trade-off} we showed that an \emph{anti}-correlation arises from the trade-off between the average log-probability $\frac{\log \plm(\str)}{N}$ and average reward $\frac{\reward(\str)}{\beta N}$.
The positive correlation between probability and quality at the level of strings, reversed at the level of corpora, is 
precisely an instance of Simpson's paradox.\looseness=-1

\section{Experimental Setup}
We conduct two experiments in the empirical portion of this paper. 
First, in \Cref{sec:toy} we validate the predictions of \Cref{prop:trade-off} with toy language and reward models.
Then, in \Cref{sec:langmodelexperiments} we demonstrate that this trade-off exists in practice for open-sourced RLHF-tuned models and that globally normalized sampling adaptors can control where on this trade-off the corpus of generated text will lie.
We also examine models aligned with Direct Preference Optimization \citep[DPO;][]{rafailov2023direct} in \Cref{sec:langmodelexperiments-dpo}, as RLHF and DPO have the same objective.\looseness=-1

\subsection{A Toy Experiment}\label{sec:toy}
\begin{figure*}[ht!]
    \centering
    \begin{minipage}[t]{0.45\textwidth}
        \centering
        \includegraphics[width=\textwidth,keepaspectratio]{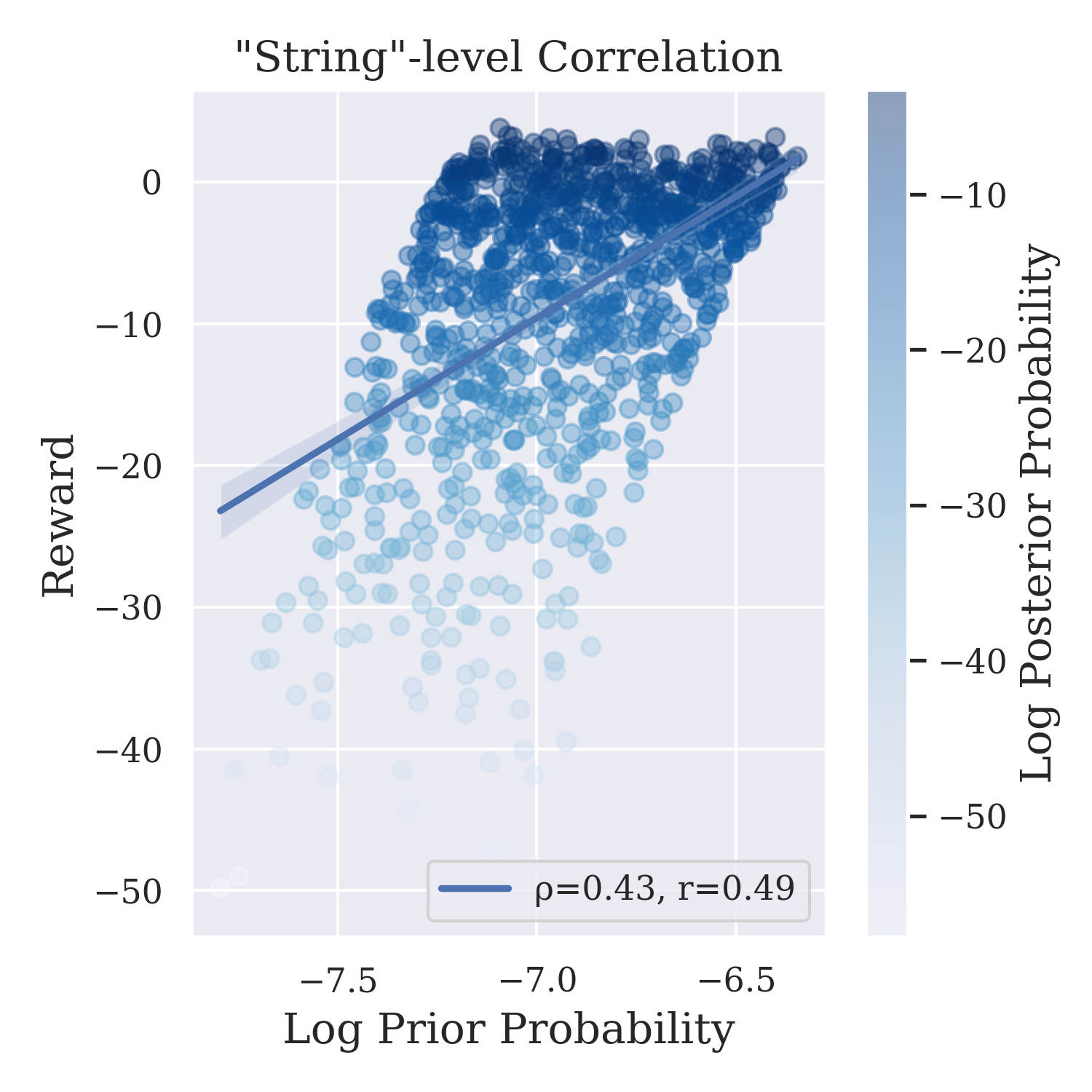}
    \end{minipage}%
    ~ 
    \begin{minipage}[t]{0.45\textwidth}
        \centering
        \includegraphics[width=\textwidth,keepaspectratio]{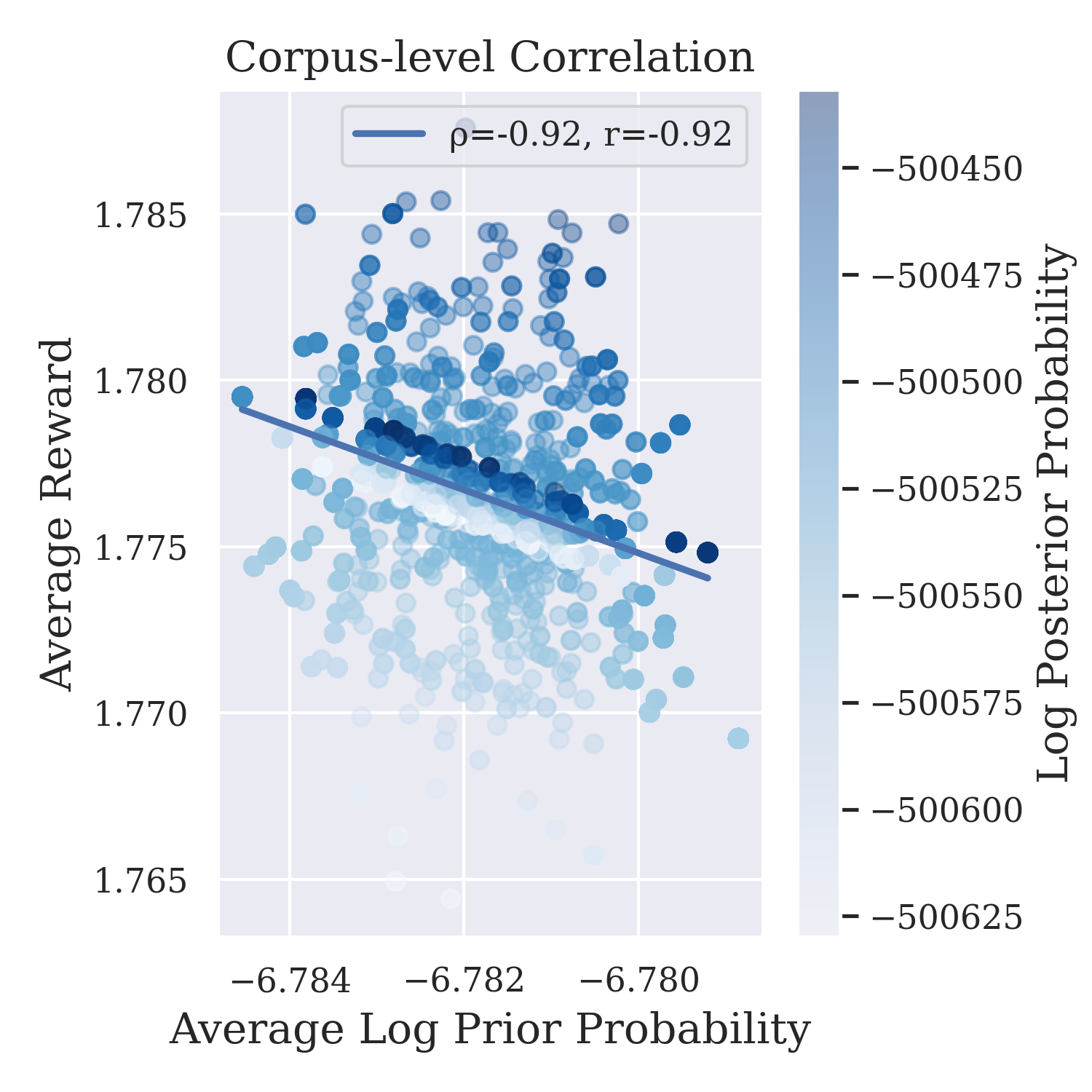}
    \end{minipage}
    \caption{Illustration of the probability--quality trade-off with toy data, where quality is measured by the reward function. (Left) ``String''-level correlations between probability and reward, where strings are mimicked by arbitrary objects. (Right) Corpus-level correlations between average log-probability and average reward. We include a best-fit line for corpora in the typical set, i.e., those with sample entropy close to $\ent(\palign)$. In both figures, the log-probability of each string or corpus is coloured according to high (dark) and low (light). }
    \label{fig:toy-simpsons}
    \vspace{-6pt}
\end{figure*}

The trade-off described in \Cref{prop:trade-off} fundamentally arises as a consequence of typicality and the fact that $\palign(\str) \propto \plm(\str) \exp(\frac{1}{\beta}\reward(\str))$.
We aim to demonstrate these theoretical principles with an easily reproducible toy experiment, where we model $\palign$ and $\plm$ as language models with support only over a finite subset of $\kleene{\alphabet}$.

\paragraph{Modeling $\palign$, $\plm$ and $\reward$.}
\naamanchange{Let $\domain \subset \kleene{\alphabet}$ be a finite set of $|\domain| = 1,000$ strings.\footnote{We arbitrarily identify $\domain$ with $\setof{1, 2, \ldots 1000}$.} 
With this, let us construct a toy aligned language model $\palign(\str)$ by sampling a distribution over $\domain$ from the Dirichlet distribution parameterized by a $|\domain|$-sized vector where every value is set to 0.1.
We create the prior $\plm(\str)$ by applying the softmax to a scaled and noised version of this distribution. 
That is, we define $\plm(\str) \defeq \mathrm{softmax}\big(\palign(\cdot)^{\frac{1}{\temp}} + \epsilon)\big)\left(\str\right)$ where $\epsilon \sim U(-\kappa, \kappa)$ and $\temp \in \Rpos$, $\kappa \in \Rpos$ are our hyperparameters. 
We then define $\reward(\str)$ analogously to \Cref{eq:secret} with $\reward(\str) = \log \frac{\palign(\str)}{\plm(\str)}$.
The distributions of $\palign(\str)$, $\plm(\str)$ and $\reward(\str)$ over the domain $\domain$ can be seen in \Cref{fig:toy}.
To induce Simpson's paradox, we tune $\kappa$ and $\temp$ such that they are positively correlated, shown in \Cref{fig:toy-simpsons} on the left.\looseness=-1}

\paragraph{Constructing Corpora with Causal Bootstrapping.}
An important point about the trade-off in \Cref{prop:trade-off} is that it occurs \emph{with high probability}.
To illustrate this, we use causal bootstrapping \citep{little2020causal} to construct corpora that are uniformly distributed across 10 bands of average log-probability under the prior.\footnote{This scheme allows us to observe low-probability corpora.}
\naamanchange{Then, we compute and visualize the corpus probabilities, i.e., $\log \palign(\strset)$ where $\strset \defeq \setof{ \str^{(n)}}_{n=1}^{N}$.} 
Due to \Cref{prop:trade-off}, we expect to see that corpora exhibiting the trade-off have much higher probability than those that do not.
We examine 1,000 corpora sampled this way, each with 100k samples.\looseness=-1

\subsection{The Trade-off in Practice}\label{sec:langmodelexperiments}
Here we demonstrate the existence of the probability--quality trade-off with an open-sourced aligned language model based on the Llama 2 family \citep{touvron2023llama}.
Using locally normalized sampling adaptors, we sample a corpus $\strset$ of 2,000 texts from an RLHF-tuned model $\qalign$.
Towards this, we randomly choose 1,000 prompts using the helpfulness dataset from \citet{perez-etal-2022-red} and for each prompt, we produce two generations.
Then, for every string in this corpus, we obtain its log-probability under the prior language model $\log \plm(\str)$ and its reward $\reward(\str)$. 
The prior and reward models are the same as those used to train $\qalign$ in an RLHF scheme.
We repeat this using five locally normalized sampling adaptors at five temperatures, totaling 25 sampling schemes and thus $50,000$ $(\log \plm(\str), \reward(\str))$ pairs.
To observe the trade-off, we compute the Pearson and Spearman's correlation between $\log \plm(\str)$ and $\reward(\str)$ at the string level, and between $\frac{\log \plm(\strset)}{N}$ and $\frac{\reward(\strset)}{N}$ at the corpus level.\looseness=-1

\paragraph{Resampling Corpora with the IMHA.}
To compute corpus-level correlations we require a lot of data points of the average log-probability and average reward. 
However, because sampling multiple corpora is prohibitively expensive, we use standard bootstrap resampling \citep{efron1994bootstrap} to create multiple corpora for each of the 25 sampling schemes.
Given a corpus of strings generated from $\qalign$ with a locally normalized sampling adaptor defined by the transform function $\globalsamplingadaptor$, we resample uniformly with replacement $N$ times, accepting and rejecting each as described in \Cref{alg:imh}.
This gives us a resampled corpus $\strset'$.
Then, we compute the average log-likelihood $\frac{\log \plm(\strset')}{N}$ and average reward $\frac{\reward(\strset')}{N}$.
We do this 2,000 times per sampling scheme, giving us a total of 50,000 $\mleft(\frac{\log \plm(\strset)}{N},\frac{\reward(\strset')}{N}\mright)$ pairs, which we then use to we compute the corpus-level correlations.
We set $N = 200,000$ as preliminary experiments showed that for $N \ge 200,000$ the IMHA converges, i.e., the autocorrelation measured with Cram{\'e}r's V falls to $< 0.10$.\looseness=-1

\begin{figure*}[ht!]
    \centering
    \includegraphics[width=\textwidth,keepaspectratio]{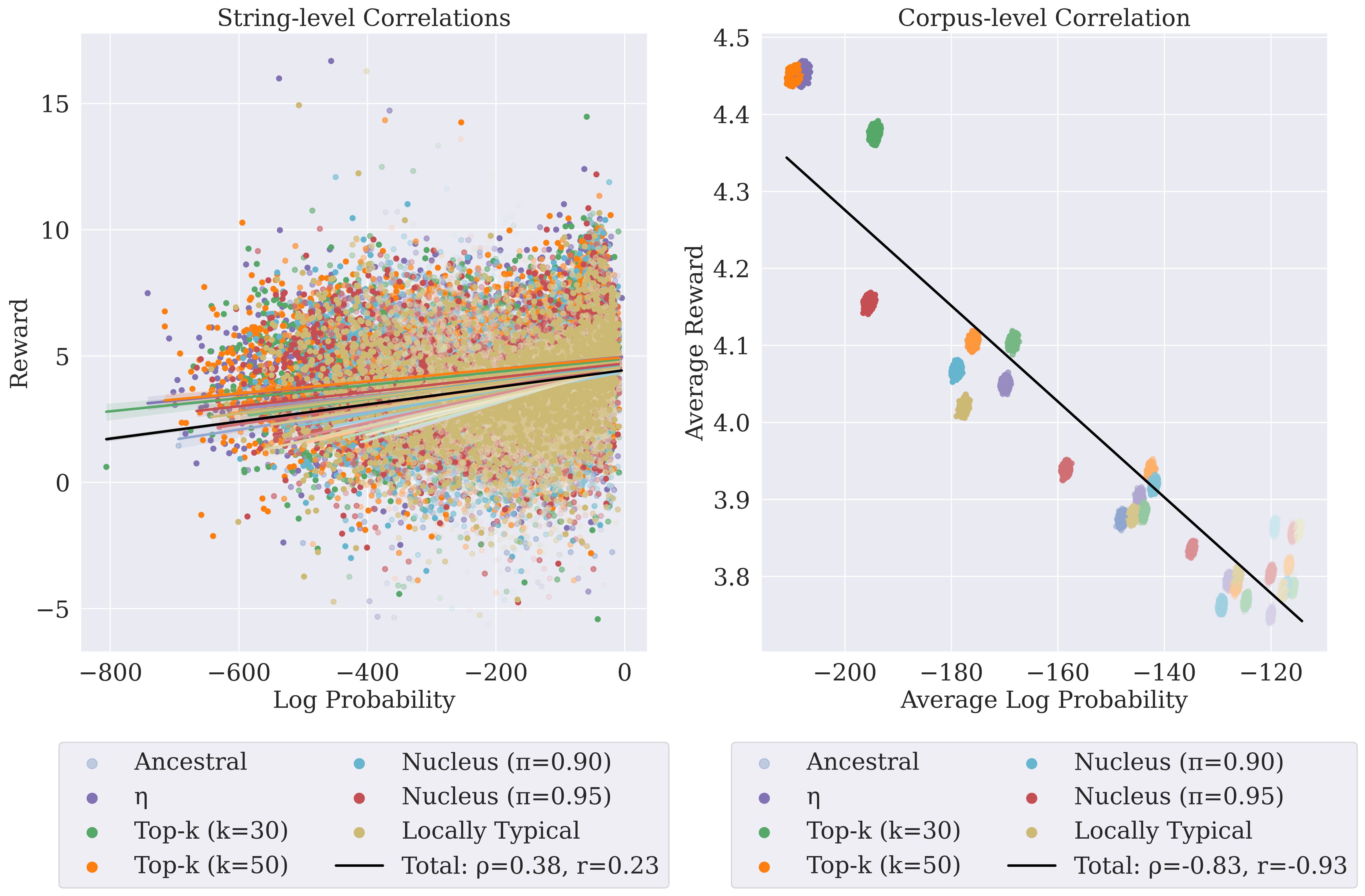}
    \caption{The probability--quality relationship, where quality is measured by the reward function. (Left) String-level correlations between log-probability and quality. (Right) Corpus-level correlations between average log-probability and average quality, with corpora created by different sampling adaptors. Higher intensity of the colours denote higher temperatures used with the sampling adaptor.\looseness=-1}
    \label{fig:simpsons}
\end{figure*}

\paragraph{Sampling adaptors.}
The five sampling schemes we examine are: top-$\topk$ sampling \citep{fan-etal-2018-hierarchical} for $\topk \in \setof{30, 50}$, nucleus sampling \citep{holtzman2020curious} for $\topp \in \setof{0.90, 0.95}$, $\eta$-sampling \citep{hewitt-etal-2022-truncation} and locally-typical sampling \citep{meister-etal-2023-locally}. 
For each setting, we examine five temperatures $\temp \in \setof{0.5, 0.75, 1.0, 1.25, 1.5}$.
This gives us a total of 25 settings that cover various real-world use cases.
As a baseline, we also include ancestral sampling.\looseness=-1

\paragraph{Models.}
We utilize the family of 7B reward, RLHF-tuned and prior language models from \citet{rando2024universal} based on Llama 2 7B \citep{touvron2023llama}.
Specifically, we use the baseline reward and RLHF-tuned models trained on the helpfulness dataset from \citet{perez-etal-2022-red}.\looseness=-1

\section{Results}\label{sec:results}
Our results confirm our theoretical findings in \Cref{sec:bayesian}.

\paragraph{A Strong Anti-correlation.}
In both toy and empirical settings, we observe at the corpus level a strong linear anti-correlation between the average log-probability $\frac{\log \plm(\strset)}{N}$ and reward $\frac{\reward(\strset)}{N}$.
In the toy experiment, as shown in \Cref{fig:toy-simpsons} corpora in the typical set
have average log-probabilities and average rewards that exhibit a Pearson correlation of $\pearson=-0.92$. 
And, importantly, these typical corpora occur with significantly higher probability than those that do not. 
For example, the median log-probability difference between a corpus in and out of the typical set is $10^{84}$ fold.\footnote{$-500,523$ vs. $-500,718$; the difference is somewhat masked by the log-scale.\looseness=-1}
The results in the empirical experiment with real language models are generally consistent with the toy experiment.
We observe a Pearson correlation between the average log-probability and average reward of $\pearson=-0.93$ with relatively few outliers.\looseness=-1

\paragraph{Globally Normalized Sampling Adaptors Control the Trade-off.}
We observe in \Cref{fig:simpsons} that corpora sampled with various globally normalized sampling adaptors are centered at different points on the trade-off and follow qualitatively expected trends.
For example, corpora sampled with different temperatures have average log-probability that follow the expected trend---high-temperature corpora have lower average log-probability and higher average reward.
And, at $\temp=1.0$ all sampling adaptors produce corpora with higher average log-probability and lower average reward than ancestral sampling.
These results are expected since lower temperature and the examined sampling adaptors skew the sampling distribution towards high probability strings.
The behaviour when comparing sampling adaptors with different degrees of truncation also follows our expectation, e.g., corpora sampled with nucleus sampling for $\topp = 0.95$ have lower average reward than those sampled with $\topp = 0.90$.
These findings are in line with our theoretical exposition in \Cref{sec:trade-off-and-adaptors} and suggest that we can use sampling adaptors to control the average reward of sampled corpora.\looseness=-1

\paragraph{Local vs. Global Normalization.}
In practice, we find that the tradeoff can be controlled by directly using locally normalized sampling adaptors, i.e., without applying the IMHA post-hoc to derive a globally normalized sampling adaptor. 
First, we see similar results when we repeat the experiment in \Cref{sec:langmodelexperiments} without applying the IMHA's acceptance-rejection protocol. 
Second, for every locally normalized sampling adaptor over $95\%$ of the samples are accepted by the IMHA.
An acceptance rate of $100\%$ implies that the proposal distribution (in this case, the distribution induced by the locally normalized sampling adaptor) is equal to the target distribution \citep{metropolis1953, hastings1970,deonovic2017convergence,wang2022metropolis}.

\paragraph{Simpson's Paradox.}
We observe in both toy (\Cref{fig:toy-simpsons}) and empirical data (\Cref{fig:simpsons}) the emergence of Simpson's paradox.
At the string-level, we measure rank correlations of $\spearman=0.43$ in the toy setting $\spearman=0.38$ in the empirical setting. 
In the latter case, this positive correlation is probably explained by the fact that the reward model is trained to model preferences of \emph{helpfulness} and the prior Llama 2 7B model has likely seen related texts in its training data.
In both settings, we simultaneously find an anti-correlation at the corpus level between average log-probability and average reward. 
These results are consistent with our expectations in \Cref{sec:simpsons}---the reversal emerges because the trade-off arises out of typicality, independently of the true relationship between probability and quality at the string level.\looseness=-1

\section{Conclusion}
Our work examines the relationship between probability and reward in sampling from RLHF-tuned language models.
We provide a formal argument and empirical evidence of a trade-off between these two quantities when generating text at scale.
Notably, this trade-off exists as a consequence of typicality, is independent of the relationship between reward and probability at the string level, and applies to \emph{any} conditionally aligned language model, not just those aligned with RLHF.

Moreover, we have proposed globally normalized sampling adaptors, and demonstrate their utility for selecting how much likelihood we exchange for reward.
We also find that locally normalized sampling adaptors are good approximations of their globally normalized counterparts in practice, and can be directly used to control the trade-off.
Altogether, these findings present a new direction of research for improving reward alignment or mitigating reward overfitting in RLHF-tuned models, and the development of sampling methods for conditional text generation.\looseness-1

\section*{Limitations}
There are three main limitations to our work.
First, is that we only conduct empirical analysis for English and Transformer-based language models.
Second, we don't experiment over all sampling adaptors, e.g., we did not consider Mirostat-sampling \citep{basu2021mirostat} or contrastive search decoding \citep{su2022a} in our experiments. 
These choices were made because the theory holds independently of these factors, though further work should consider other model architectures, sampling adaptors and models that span a variety of languages and domains.
Finally, we have only examined the probability--quality relationship under the paradigm of RLHF (and equivalently, DPO, as we show in \Cref{sec:langmodelexperiments-dpo}), but not other alignment methods like ORPO \citep{hong2024orpo}, ODPO \citep{amini2024directpreferenceoptimizationoffset} or variational BoN \citep{amini2024variationalbestofnalignment}. 
We leave those to future work.\looseness=-1

\section*{Ethics Statement}
We do not foresee any ethical implications.

\section*{Acknowledgments}
This research is supported by the Singapore Ministry of Education Academic Research Fund Tier 1 (T1 251RES2216).
Josef Valvoda is funded by the Nordic Programme for Interdisciplinary Research Grant 105178 and the Danish National Research Foundation Grant no. DNRF169.
Anej Svete is supported by the ETH AI Center Doctoral Fellowship.

\bibliography{custom}

\appendix

\onecolumn

\section{Locally Normalized Sampling Adaptors}\label{app:top-k-example}
Despite their empirical success, the normalization performed in locally normalized sampling adaptors can lead to strings that are scored higher by the transform function $\globalsamplingadaptor$ to have lower probability under the induced language model $\plmlocalsampling$. 
This behaviour stems from the fact that the normalization is dependent on the weights assigned to \emph{other} symbols at a given time step, and thus leads to inconsistencies at the global level. 
This is arguably undesirable since it makes it difficult to formally reason about how the transform function---which embodies the core logic of the sampling adaptor---influences the properties of strings sampled from $\plmlocalsampling$. 
To make this clear, we provide an example of this behaviour in top-$\topk$ sampling \citep{fan-etal-2018-hierarchical}.\looseness=-1
\begin{example}
Consider an alphabet $\eosalphabet = \setof{a,b,c,\eos}$. 
Let us define a language model\footnote{Note that this language model is not necessarily tight.} such that 
\begin{equation}
    \plm(\str) = \plm(\eos \mid \str) \prodstring \plm(\tokent \mid \strltt)
\end{equation}
and 
\begin{equation}
    \plm(\cdot \mid \strltt) =
    \begin{cases}
        [0.4, 0.1, 0.1, 0.4] & \text{if $\token_{t-1} = a$} \\
        [0.1, 0.4, 0.2, 0.3] & \text{if $\token_{t-1} = b$} \\
        [0.5, 0.5, 0.0, 0.0] & \text{if $t = 1$} \\
    \end{cases}
\end{equation}
where the vector $[\dots]$ denotes the probability assigned to $a,b,c,\eos$, in that order.
Let us now consider the probability of the strings $aa$ and $bb$. 
\begin{align}\label{eq:pre-topk-probs}
    \plm(aaa) = 0.5 \times 0.4 \times 0.4 \times 0.4 = 0.032 \\
    \plm(bbb) = 0.5 \times 0.4 \times 0.4 \times 0.3 = 0.024 
\end{align}
Let us now consider their probability when top-$2$ sampling is applied. 
\begin{align}\label{eq:post-topk-probs}
    \plmlocalsampling(aaa) = 0.5 \times 0.5 \times 0.5 \times 0.5 = 0.0625\\
    \plmlocalsampling(bbb) = 0.5 \times \frac{0.4}{0.7} \times \frac{0.4}{0.7} \times \frac{0.3}{0.7} = 0.06997
\end{align}
Because the transform function in top-$\topk$ does not modify the symbol probability if the symbol is kept (see \Cref{app:sampling-adaptor-examples}), \Cref{eq:pre-topk-probs} is precisely the score assigned to the strings by the transform function, and \Cref{eq:post-topk-probs} is the score after the normalization step. 
And we observe a reversal---the string ``$aa$'' has a higher score than ``$bb$'', but is later assigned a lower probability under the induced model $\plmlocalsampling$.
\end{example}

\section{Supplementary proofs for \Cref{sec:bayesian}}
\subsection{Proof of \Cref{eq:chebyschev}}
\label{app:proof-chebyshev}
\begin{proof}
\begin{subequations}
    \begin{align}
        \prob\mleft( \strset \notin  \typicalset(\qalign)\mright) &= \prob( \Big|  \ent(\rvY, \rvA = \good) + \frac{\log \qalign(\strset)}{N} \Big| \ge \varepsilon )  \\
        &= \prob( \Big|  \ent(\rvY, \rvA = \good) - \frac{- \log \qalign(\strset)}{N} \Big| \ge \varepsilon) \\
        &\le \frac{\var( \frac{- \log \qalign(\strset)}{N})}{\varepsilon^2}
        = \frac{\var(-\log \qalign(\str))}{N \varepsilon^2} = \frac{\var(\rvI)}{N \varepsilon^2} \label{eq:apply-chebyshev}
    \end{align} %
\end{subequations}
\Cref{eq:apply-chebyshev} holds due to Chebyshev's inequality.
\end{proof}

\subsection{Proof of the Probability--Quality trade-off}
\label{app:proof-trade-off}
\begin{proof}
    Consider the $(N, \varepsilon)$-typical set 
    \begin{align} 
        \typicalset(\qalign) \defeq \Big\{ \strset \in (\kleene{\alphabet})^N   
    \mid \Big|  \ent(\rvY, \rvA = \good) + \frac{\log \qalign(\strset)}{N} \Big| < \varepsilon \Big\} 
    \end{align}
    By rewriting \Cref{eq:secret}
    \begin{align*}
        \frac{\reward(\str)}{\beta} = \log \frac{\qalign(\str)}{\plm(\str)} + \log \normalizingconstantrlhf 
    \end{align*}
    as
    \begin{align*}
        \log \qalign(\str) = \frac{\reward(\str)}{\beta} + \log \plm(\str) - \log \normalizingconstantrlhf,
    \end{align*}
     and summing over all $\str \in \strset$, we get
    \begin{align} \label{eq:rewrite-secret}
        \log \qalign(\strset) = \frac{\reward(\strset)}{\beta} + \log \plm(\strset) - N \log \normalizingconstantrlhf,
    \end{align}
    
    Substituting \Cref{eq:rewrite-secret} into \Cref{eq:proof-trade-off-typical-set}, we obtain
    \begin{align}
        \typicalset(\qalign) &= \Big\{ \strset \in (\kleene{\alphabet})^N   
    \mid \Big|  \ent(\rvY, \rvA = \good) + \frac{\reward(\strset)}{N\beta} + \frac{\log \plm(\strset)}{N} - \log \normalizingconstantrlhf \Big| < \varepsilon \Big\} \nonumber \\
    &= \Big\{ \strset \in (\kleene{\alphabet})^N   
    \mid \Big| C + \frac{\log \plm(\strset)}{N}  + \frac{\reward(\strset)}{N\beta} \Big| < \varepsilon \Big\}  \nonumber 
    \end{align}
    Due to Chebyshev's inequality, $\prob\mleft( \strset \notin T_N^{\varepsilon}(\qalign)\mright) \leq \frac{\var(\rvI)}{N \varepsilon^2}$, we have $\Big| C + \frac{\log \plm(\strset)}{N}  + \frac{\reward(\strset)}{N\beta} \Big| < \varepsilon$ with probability at least $\mleft( 1 - \frac{\var(\rvI)}{N \varepsilon^2} \mright)$ for all $N$ and $ \varepsilon > 0$.
    When $N \ge \frac{\var(\rvI)}{ \delta \varepsilon^2}$, the above holds  with probability at least $1 - \delta$.\looseness=-1
\end{proof}

\section{Infinite-Entropy Language Models}\label{sec:inf-entropy-lms}
A key assumption we have made in this paper is that all language models $\plm$ under consideration have \emph{finite} entropy.
In general, this is not true.
To make this point clear, we give an example of a simple language model whose entropy diverges.
\begin{example}[A Tight LM with Infinite Entropy]
    Let $\alphabet \defeq \{a\}$ and define for $t = 1, 2, \ldots$
    \begin{equation} \label{eq:infinite-entropy-lm}
        \plm(\underbrace{a \cdots a}_{\textnormal{$t$ times}}) \defeq \frac{1}{\lg\left(t + 1\right)} - \frac{1}{\lg\left(t + 2\right)}.
    \end{equation}
\end{example}
\begin{proposition} \label{prop:infinite-entropy}
   The language model $\plm$ from \Cref{eq:infinite-entropy-lm} is tight and has infinite entropy.
\end{proposition}
\begin{proof}
    The proof follows \citet{baer-infinite-entropy}.
    We consider the language model:
    \begin{equation}
        \plm(\underbrace{a \cdots a}_{\textnormal{$t$ times}}) \defeq \frac{1}{\lg\left(t + 1\right)} - \frac{1}{\lg\left(t + 2\right)}
    \end{equation}    
    $\plm(\underbrace{a \cdots a}_{\textnormal{$t$ times}})$ is positive over $t = 1, 2, \ldots$ and sums to $1$ since it forms a telescoping sum with the only remaining term $\frac{1}{\lg\left(2\right)} = 1$.
    This proves that $\plm$ is tight.
    
    Furthermore, we can show that $\plm$'s entropy is $\infty$.  
    Let us denote $\plm_t \defeq \plm(\underbrace{a \cdots a}_{\textnormal{$t$ times}})$, and begin by pointing out several facts. 
    First, the monotonicity and convexity of $\frac{1}{\lg x}$ is easily seen by noting that its first derivative is $-\frac{1}{x \lg^2 x \log 2}$ (negative for $x > 1$) and its second derivative is $\frac{\lg(x + \frac{2}{\log 2})}{x^2 \lg^3x \log(2)}$ (positive for $x > 1$). 
    This will allow us to bound $\plm_t$ from below with $\plm_t' \defeq \frac{1}{t \lg^2 t \log 2}$, which is monotonically decreasing and less than $\frac{1}{2}$ for $t \ge 3$.
    Then, we point out that with basic calculus we can see that $\plm \lg \plm$ is monotonically decreasing with $\plm$ for $\plm < \frac{1}{2}$.
    With these, we can say that $\plm_t' \lg \plm_t'$ is monotonically decreasing for $t \ge 2$ since $\plm_t' < \frac{1}{2}$ for these $t$.
    We are now ready to lower bound $\ent(\plm)$ with an expression equal to infinity, thereby showing that $\ent(\plm)$ is infinite:
    \begin{subequations}
        \begin{align}
            \ent\left(\plm\right) 
            &= -\sum_{t = 1}^{\infty} \plm_t \lg \plm_t \\
            &= -\plm_1 \lg \plm_1 - \plm_2 \lg \plm_2 - \sum_{t = 3}^{\infty} \plm_t \lg \plm_t \\
            &= -\plm_1 \lg \plm_1 - \plm_2 \lg \plm_2 - \sum_{t = 3}^{\infty} \mleft(\frac{1}{\lg\left(t + 1\right)} - \frac{1}{\lg\left(t + 2\right)}\mright) \lg \mleft(\frac{1}{\lg\left(t + 1\right)} - \frac{1}{\lg\left(t + 2\right)}\mright) \\
            &> \frac{1}{\log 2} \sum^\infty_{t=3}\frac{1}{t \lg^2 t}(\lg t + \lg \lg^2 t + \lg\log 2) \\
            &> \frac{1}{\log 2} \sum^\infty_{t=3}\frac{1}{t \lg^2 t} \\
            &> \frac{1}{\log 2} \int^\infty_3 \frac{1}{t \lg t} dt \\
            &> \lim_{n\rightarrow\infty} (\lg2) (\lg \lg n - \lg \lg 3) = \infty
        \end{align}
    \end{subequations}
\end{proof}

\section{A Tighter (Chernoff) Bound}\label{sec:chernoff}
In this section, we give a tighter concentration inequality than the (standard) one derived with Chebyshev's inequality. 
The inequality displayed in \Cref{eq:chebyschev} is weak in the sense that the average right hand size is $\bigo(\frac{1}{N})$---ideally, we desire a concentration inequality that is exponential, i.e., $\bigo(\exp(-cN))$ for some constant $c \in \Rpos$. 
To prove such a tighter concentration inequality, we define a class of language models that we term \defn{sub-exponential} language models.
We show that both classical $n$-gram language models as well as modern Transformer-based language models are sub-exponential under our definition.
We further show that we can apply the Chernoff--Cram{\'e}r method to argue that the sample entropy collapses around the mean exponentially quickly.

Before we delve into our derivation, we highlight what makes a direct application of a standard concentration bound, e.g., a Hoeffding bound, tricky.
Consider a language model $\plm$ with support everywhere on $\kleene{\alphabet}$.
Furthermore, consider an enumeration $\setof{\str_n}_{n=1}^\infty$ of $\kleene{\alphabet}$ such that $n > m$ implies $\plm(\str_n) \leq \plm(\str_m)$.
Observing the infinite sum $\sum_{n=1}^\infty \plm(\str_n) = 1$ is convergent, we must have that $\plm(\str_n) \rightarrow 0$ as $n \rightarrow \infty$. 
It follows by the continuity of $\log$, that $-\log \plm(\str_n) \rightarrow \infty$ as $n \rightarrow \infty$. 
A simpler way of stating the above is that
the random variable $\rvI(\str) = - \log \plm(\str)$, distributed according to\looseness=-1
\begin{equation}
    \prob(\rvI = \information) = \sumoverstrings \plm(\str) \indicatorf \{\information = - \log \plm(\str)\},
\end{equation}
is unbounded.

\subsection{Prerequisites}\label{sec:renyi}
We will now introduce several definitions and prove several results.

\begin{definition}[Non-trivial Language Model]
We call a language model over $\alphabet$ \defn{non-trivial} if its support is an infinite subset of $\kleene{\alphabet}$.
\end{definition}

\begin{definition}[\Renyi Entropy]
Let $\plm$ be a language model over $\alphabet$.
The \defn{\Renyi entropy} of $\plm$ is defined as
\begin{equation}
\begin{aligned}
    \renyient(\plm) = \begin{cases}
        \frac{1}{1-\renyialpha} \log \sumoverstrings \plm(\str)^\renyialpha  &  \renyialpha \in (0, 1) \\
     - \sumoverstrings \plm(\str) \log  \plm(\str)  & \renyialpha = 1
    \end{cases}
\end{aligned}
\end{equation}
for $\renyialpha \in (0, 1]$
\end{definition}

\begin{definition}
A language model $\plm$ is $\eos$-bounded if there exists $c$ such that $\plm(\eos \mid \str) > c > 0$ for all $\str \in \kleene{\alphabet}$.\footnote{
Strictly speaking, autoregressive language models are not always language models, i.e., valid probability distributions over $\kleene{\alphabet}$ where $\sumoverstrings \plm(\str) = 1$. We highlight this difference in our exposition on locally vs. globally normalized sampling adaptors in \Cref{sec:sampling-adaptors}\looseness=-1.}
In other words, $\plm(\str \oplus \eos) \le (1-c)^{|\str|}$ for all $\str \in \kleene{\alphabet}$.
\end{definition}

\begin{restatable}{proposition}{boundedrenyi}\label{prop:boundedrenyi}
Let $\plm$ be an $\eos$-bounded language model. 
Then, $\renyient(\plm) < +\infty$ for $\renyialpha \in (0, 1]$.
\end{restatable}
\begin{proof}
We divide the proof into two cases.

\noindent \textbf{Case 1: $\renyialpha \in (0, 1)$.}
Consider the following manipulation
\begin{subequations}
\begin{align}
\sumoverstrings \plm(\str)^{\renyialpha} &= \sum_{n=0}^\infty \sum_{\str \in \alphabet^n} \plm(\str)^{\renyialpha} \\
&\leq \sum_{n=0}^\infty (1-c)^{n\renyialpha} \\ 
&= \sum_{n=0}^\infty \left[(1-c)^{\renyialpha}\right]^n \\
&= \frac{1}{1 -(1-c)^{\renyialpha}} < +\infty
\end{align}
\end{subequations}
The last inequality follows because $c \in (0, 1)$, and, thus, we have $0 < (1-c)^{\renyialpha} < 1$ and, thus, the geometric sum converges.

\noindent  \textbf{Case 2: $\alpha=1$.}
In the case of $\alpha=1$, \Renyi entropy entropy turns into Shannon entropy.
Because $\log(\cdot)$ is concave, we have
\begin{subequations}
\begin{align}
    -\sumoverstrings& \plm(\str) \log  \plm(\str) = \log \prod_{\str \in \kleene{\alphabet}} \frac{1}     {\plm(\str)^{\plm(\str)}}  \\
    &= 2 \log \prod_{\str \in \kleene{\alphabet}} \mleft(\frac{1}{\plm(\str)^{\frac{1}{2}}}\mright)^{\plm(\str)}  \\
    &\leq 2 \log \mleft(\sumoverstrings \frac{\plm(\str)}{\plm(\str)^{\frac{1}{2}}} \mright)  \label{eq:case2-gmleam} \\
    &= 2 \log \mleft(\sumoverstrings \plm(\str)^{\frac{1}{2}} \mright)  \\
    &= \ent_{\frac{1}{2}}(\plm)  \\
    &< +\infty
\end{align}
\end{subequations}%
\Cref{eq:case2-gmleam} holds due to GM--AM inequality $ \prod_{i} x_i^{p_i} \le \sum_i p_i x_i $, when $\sum_i p_i = 1$ and $x_i > 0 \ \forall i$.

\end{proof}

\begin{corollary}
Let $\plm$ be a Transformer-based language model. 
Then, $\ent_{\renyialpha}(\plm) < +\infty$ for $\renyialpha \in (0, 1]$.
\end{corollary}\label{cor:eos-bounded-transformer}
\begin{proof}
    This follows from the proof in \citet[Prop. 4.7 and Thm. 5.9,][]{du-etal-2023-measure} that Transformer-based LMs are $\eos$-bounded.
\end{proof}

\begin{corollary}
Let $\plm$ be a tight $n$-gram language model. 
Then, $\ent_{\renyialpha}(\plm) < +\infty$ for $\renyialpha \in (0, 1]$.
\end{corollary}
\begin{proof}
Tight $n$-gram LMs are trivially are $\eos$-bounded.
\end{proof}

\newpage
\begin{restatable}{proposition}{monotonicity} \label{prop:monotonicity}
Let $\plm$ be an $\eos$-bounded language model.
Then, over the interval $(0, 1]$, 
$\ent_{\renyialpha}(\plm)$ is monotonically decreasing in $\renyialpha$. 
Moreover, if $\plm$ is a non-trivial language model, then $\ent_{\renyialpha}(\plm)$ is \emph{strictly} monotonically decreasing in $\renyialpha$.
\end{restatable}

\begin{proof}
Consider the following derivation
\begin{subequations}
    \begin{align}
        \diff{ \renyient(\plm)}{\renyialpha} =&  \frac{1}{(1-\renyialpha)^2} \log \sumoverstrings \plm(\str)^\renyialpha +  \frac{1}{1-\renyialpha} \sumoverstrings \frac{\plm(\str)^\renyialpha \log \plm(\str)}{\sumoverstringsprime \plm(\strprime)^\renyialpha} \\
        =& \frac{1}{(1-\renyialpha)^2} \sumoverstrings z(\str)  \mleft( \log \sumoverstringsprime \plm(\strprime)^\renyialpha \mright) +  \frac{1}{(1-\renyialpha)} \sumoverstrings  z(\str) \log \plm(\str)  \\
        =& \frac{1}{(1-\renyialpha)^2} \sumoverstrings z(\str)  \mleft( \log \sumoverstringsprime \plm(\strprime)^\renyialpha \mright) +  \frac{1}{(1-\renyialpha)^2} \sumoverstrings  z(\str) \log \plm(\str)^{1-\renyialpha}  \\
         = &\frac{1}{(1-\renyialpha)^2} \sumoverstrings z(\str) \mleft(\log \sumoverstringsprime \plm(\strprime)^\renyialpha + \log \plm(\str)^{1-\renyialpha} \mright)  \\
         = &\frac{1}{(1-\renyialpha)^2} \sumoverstrings z(\str) \mleft( - \log \frac{\plm(\str)^\renyialpha}{\sumoverstringsprime \plm(\strprime)^\renyialpha} + \log \plm(\str) \mright)  \\
         = & \frac{1}{(1-\renyialpha)^2} \sumoverstrings z(\str) \log \frac{\plm(\str)}{z(\str)}  \\
         = & - \frac{1}{(1-\renyialpha)^2}  \KL(z \,\|\, \plm) \le 0,
    \end{align} 
    \end{subequations}
    where $z(\str) \defeq \frac{\plm(\str)^\renyialpha}{\sumoverstrings \plm(\str)^\renyialpha}$.
    Because the derivative of $\renyient(\plm)$ with regard to $\renyialpha$ is $\leq$ 0 on the interval $ (0, 1]$, $\renyient(\plm)$ is monotonically decreasing in $\renyialpha$.
    Moreover, when $\plm$ is non-trivial, which implies $\plm$ is not uniform nor a point mass\footnote{i.e., the size of the support of $\plm$ is 1.}, we have $z \neq \plm , \forall \renyialpha \in (0, 1)$. 
    Thus $\diff{ \renyient(\plm)}{\renyialpha} = - \frac{1}{(1-\renyialpha)^2}\KL(z \,\|\, \plm) < 0 $, i.e., $\renyient(\plm)$ is strictly monotonically decreasing on $(0, 1)$.
\end{proof}

\newpage
\begin{restatable}{proposition}{varrvIisbounded} \label{prop:varrvibounded}
Let $\plm$ be an $\eos$-bounded language model.
Then, $\var(\rvI)$ is finite.
\end{restatable}
\begin{proof}
We show that $\var(\rvI)$ is bounded for $\eos$-bounded language models.
Let $M \defeq \sumoverstrings \plm(\str)^{\frac{1}{2}}, z(\str) \defeq \frac{\plm(\str)^{\frac{1}{2}}}{M}$. 
Note that $ M = \exp(\frac{1}{2} \ent_{\frac{1}{2}}(p)) < \infty$ due to \Cref{prop:boundedrenyi}.
Then, we have 
\begin{subequations}
\begin{align}
    \var(\rvI) &= \sumoverstrings \plm(\str) \mleft(\log  \frac{1}{\plm(\str)} \mright)^2 - \ent(\plm)^2\\ 
    &\le \mleft( \sumoverstrings \plm(\str)^{\frac{1}{2}} \log  \frac{1}{\plm(\str)} \mright)^2 - \ent(\plm)^2 \\ 
    &= \mleft( 2 \sumoverstrings \plm(\str)^{\frac{1}{2}} \log\mleft(\frac{1}{\plm(\str)}\mright)^{\frac{1}{2}} \mright)^2 - \ent(\plm)^2\\  
    &= \mleft( 2 \sumoverstrings M z(\str) \log \frac{1}{M z(\str)} \mright)^2 - \ent(\plm)^2\\ 
    &= \mleft(-2M \log M + 2 \sumoverstrings M z(\str) \log \frac{1}{z(\str)} \mright)^2 - \ent(\plm)^2\\ 
    &= \mleft(-2M \log M + 2M \ent_{1}(z)\mright)^2 - \ent(\plm)^2\\ 
    &\le \mleft(|2M \log M| + |2M \ent_{1/2}(z)| \mright)^2 - \ent(\plm)^2  \qquad \text{(monotonicity of Rényi Entropy)}\\ 
    &= \mleft(2M \log M + 4M  \log\mleft(\sumoverstrings z(\str)^{1/2} \mright) \mright)^2 - \ent(\plm)^2 \\ 
    &= \mleft(2M \log M + 4M \mleft(\frac{3}{4} \ent_{1/4}(p) - \frac{1}{2} \log M \mright)\mright)^2 - \ent(\plm)^2 \\ 
    &< +\infty
\end{align}
\end{subequations}
\end{proof}

\begin{definition}[\Renyi Gap]
Let $\plm$ be a language model and let $\renyialpha \in (0, 1]$.
The \defn{\Renyi gap} is defined as\looseness=-1
\begin{equation}
    \gapalpha(\plm) = \renyient(\plm) - \ent(\plm)
\end{equation}
\end{definition}
\begin{corollary}
Let $\plm$ be a language model and let $\alpha \in (0, 1]$.
Then, the \Renyi gap $\gapalpha(\plm)$ is non-negative.
\end{corollary}
\begin{proof}
This follows from \Cref{prop:monotonicity}.
\end{proof}

\begin{lemma}\label{lemma:continuity}
Let $\plm$ be a non-trivial, $\eos$-bounded language model.
Then, for any $ \varepsilon > 0$, there exists an $\renyialpha \in (0, 1)$ such that the \Renyi gap $0 < \gapalpha < \varepsilon$.
\end{lemma}
\begin{proof}
This follows from the $\gapalpha$ being a continuous monotonically decreasing function in $\renyialpha$ and $ \gapalpha = 0$ when $\renyialpha = 1$. %
\end{proof}

\newpage
\subsection{A Tighter Concentration Bound}
We now introduce a sharper version of the AEP for $\eos$-bounded language models.
As shown in \Cref{cor:eos-bounded-transformer}, this includes Transformer-based language models, which constitute the base architecture for most modern models \citep{brown2020,touvron2023llama}. 
The theorem is stated below.

\begin{restatable}{theorem}{chernoff} \label{thm:chernoff}
    Let $\plm$ be an $\eos$-bounded, non-trivial language model. 
    Then, there exists a function $s(\varepsilon) > 0$ such that, for any $\varepsilon > 0$, we have
    \begin{equation}
    \prob\left(\left|\frac{1}{N}\sum_{n=1}^N \rvI_n - \ent(\plm)\right| \geq \varepsilon\right) \leq 2\exp(-s(\varepsilon) N)
    \end{equation}
with 
\begin{equation}
    s(\varepsilon) \defeq -t(\varepsilon)(\gap_{1-t(\varepsilon) }(\plm)-\varepsilon)
\end{equation}
\end{restatable}

\begin{proof}
To prove the result, we apply a Chernoff bound.
This is a one-sided bound and the other will follow by symmetry.\looseness=-1
\begin{subequations}
\begin{align}
    \prob\left(\frac{1}{N}\sum_{n=1}^N\rvI_n - \ent(\plm) \geq \varepsilon\right) &\leq \inf_{t > 0} \exp(-t \varepsilon) \expected \prod_{n=1}^N 
 \exp\mleft(\frac{t}{N}(\rvI_n - \ent(\plm))\mright) \\
    &= \inf_{t > 0} \exp(-t \varepsilon) \exp(-t\ent(\plm)) \expected \prod_{n=1}^N \exp\mleft(\frac{t}{N} \rvI_n\mright) \\
    &= \inf_{t > 0} \exp(-t \varepsilon) \exp(-t\ent(\plm)) \prod_{n=1}^N  \expected \exp\mleft(\frac{t}{N} \rvI_n\mright) \\
   &= \inf_{t > 0} \exp(-t \varepsilon) \exp(-t\ent(\plm)) \prod_{n=1}^N  \sumoverstrings \plm(\str) \exp\mleft(-\frac{t}{N} \log \plm(\str)\mright) \\
    &= \inf_{t > 0} \exp(-t \varepsilon) \exp(-t\ent(\plm)) \prod_{n=1}^N  \sumoverstrings \plm(\str)^{1-\frac{t}{N}} \\
   &= \inf_{t > 0} \exp(-t \varepsilon) \exp(-t\ent(\plm)) \prod_{n=1}^N \exp\mleft(\frac{t}{N} \ent_{1-\frac{t}{N}}(\plm)\mright) \\
   &= \inf_{t > 0} \exp(-t \varepsilon) \exp(-t\ent(\plm))  \exp\mleft(t \ent_{1-\frac{t}{N}}(\plm)\mright) \\
   &= \inf_{t > 0} \exp(-t \varepsilon) \exp(t(\ent_{1-\frac{t}{N}}(\plm) - \ent(\plm))) \\
   &= \inf_{t > 0} \exp(-t \varepsilon) \exp(t(\gap_{1-\frac{t}{N}}(\plm))) \\
   &= \inf_{t > 0} \exp\mleft(t (\gap_{1-\frac{t}{N}}(\plm)-\varepsilon) \mright)\\
   &= \inf_{t' > 0} \exp\mleft(Nt' (\gap_{1-t'}(\plm)-\varepsilon) \mright)
\end{align}
\end{subequations}
Now, by \Cref{lemma:continuity}, for any $\epsilon > 0$ we can find a $0 < t(\varepsilon) < 1$ such that
$\gap_{1-t(\varepsilon)}(\plm)-\varepsilon < 0$.
Thus, we have
\begin{equation}
    \prob\left(\frac{1}{N}\sum_{n=1}^N\rvI_n - \ent(\plm) \geq \varepsilon\right) \leq  \exp(N t(\varepsilon) (\gap_{1-t(\varepsilon) }(\plm)-\varepsilon )).
\end{equation}
Similarly, we have: 
\begin{equation}
    \prob\left(\frac{1}{N}\sum_{n=1}^N\rvI_n - \ent(\plm) \leq -\varepsilon\right) \leq  \exp(N t(\varepsilon) (-\gap_{1-t(\varepsilon) }(\plm)-\varepsilon )) \leq  \exp(N t(\varepsilon) (\gap_{1-t(\varepsilon) }(\plm)-\varepsilon ))
\end{equation}
And, finally, we get:
\begin{align}
    \prob\left(\left|\frac{1}{N}\sum_{n=1}^N\rvI_n - \ent(\plm) \right|\geq \varepsilon\right) 
    &\leq  \prob\left(\frac{1}{N}\sum_{n=1}^N\rvI_n - \ent(\plm) \geq \varepsilon\right) + \prob\left(\frac{1}{N}\sum_{n=1}^N\rvI_n - \ent(\plm) \leq -\varepsilon\right) \nonumber \\
    & \leq 2\exp(N t(\varepsilon) (\gap_{1-t(\varepsilon) }(\plm)-\varepsilon )).
\end{align}
Substituting in $s(\varepsilon) = -t(\varepsilon)(\gap_{1-t(\varepsilon) }(\plm)-\varepsilon) > 0$, we arrive at
\begin{equation}
    \prob\left(\left|\frac{1}{N}\sum_{n=1}^N \rvI_n - \ent(\plm)\right| \geq \varepsilon\right) \leq 2\exp(-s(\varepsilon) N),
\end{equation}
which tends to $0$ exponentially quickly as $N \rightarrow \infty$.
Note that $2\exp(-s(\varepsilon) N)$ is $\bigo(\exp(-cN))$ for $c = s(\varepsilon)$, which proves the result.
\end{proof}

In words, with respect to Transformer-based language models, \Cref{thm:chernoff} says that if we have a model $\plm$ and randomly sample $N$ strings $\strset \sim \plm$, when we average their surprisal values we approach the entropy of $\plm$ exponentially quickly.
One caveat is that the constant in the exponential is not a universal constant, i.e., it depends on $\varepsilon$.
This is less desirable, of course, but it is an improvement over the $\bigo(\frac{1}{N})$ rate given by an application of the standard AEP.
We leave finding a universal constant for $\eos$-bounded language models to future work.

\section{Sampling Adaptors and String Probability}\label{app:global-sampling-adaptors}
In this section, we provide, by means of a simple example, an intuition of how an appropriate choice of globally normalized sampling adaptor can be used to control generated strings' average probability under the prior $\plm(\str)$.
Inspired by \citet{meister-etal-2023-efficacy}, we note that most $\globalsamplingadaptor$ can be formulated as the composition of truncation and scaling functions.
The \defn{truncation function} $\truncset\colon \kleene{\alphabet} \rightarrow \powerset(\eosalphabet)$ is a function used to find the set of symbols that meets specified criteria given the prior context, so that symbols deemed likely to lead to undesirable text can have their probability reassigned to other symbols, e.g., to only keep the top-$\topk$ symbols.
The \defn{scaling function} $\scaling \colon \Rpostoken \rightarrow \Rpostoken$ is a simple scaling of the symbol probability, e.g., to the power of $\frac{1}{\temp}$ for some temperature parameter $\temp \in \Rpos$.
With these definitions we can express a transform function $\globalsamplingadaptor$ as\looseness=-1
\begin{equation}\label{eq:adaptor-decompos}
    \globalsamplingadaptor \mleft(\plm(\cdot \mid \strltt)\mright)(\token) = \scaling(\plm(\token \mid \strltt)) \indicatorf \setof{\token \in \truncsett}.
\end{equation}
That is, given a symbol distribution $\plm(\cdot \mid \strltt)$, we apply the scaling function to scale symbol probabilities as needed and then remove symbols according to the truncation function to arrive at the output unnormalized distribution. 
For instance, we can express the transform function in nucleus sampling \citep{holtzman2020curious} with:
\begin{subequations}
    \begin{align}
        \scaling_{\mathrm{nucleus}}\big((\plm(\token \mid \strltt)\big) &= \identity\big(\plm(\token \mid \strltt)\big) \\
        \truncset_{\mathrm{nucleus}}\big((\plm(\token \mid \strltt)\big) &= \argmin_{\eosalphabet' \subseteq \eosalphabet}|\eosalphabet'| \qquad s.t. \sum_{\token \in \eosalphabet'} \plm(\token \mid \strltt) \geq \topp 
    \end{align}
\end{subequations}
where $\identity$ denotes the identity function and $\topp \in \Rpos$ is a hyperparameter.
See \Cref{app:sampling-adaptor-examples} for more examples.

Let $\plm$ be a language model and $\palign$ its aligned counterpart.
We consider the probability of a string under the induced distribution $\palignsampling$ when a globally normalized sampling adaptor is used with an aligned model $\palign$. 
Since our goal is simply to provide an intuition behind what might happen by analyzing the effects of the global sampling adaptor, we make the following simplifying assumptions.
First, we assume that the alignment does not modify the truncation function $\truncset$: $\truncset\mleft(\plm\mleft(\cdot \mid \str\mright)\mright) = \truncset\mleft(\palign \mleft(\cdot \mid \str\mright)\mright)$ for all $\str \in \kleene{\alphabet}$.
Second, we assume that $\scaling\mleft(\cdot\mright) = \identity\mleft(\cdot\mright)$.
With this, we can derive
\begin{subequations}\label{eq:adaptor-string-prob}
 \begin{align} 
   \palignsampling(\str) &\propto \globalsamplingadaptor\big(\palign(\cdot \mid \str)\big)(\eos) \prodstring \globalsamplingadaptor\big(\palign(\cdot \mid \strltt)\big)(\tokent) \nonumber \\
   &= \scaling(\palign(\eos \mid \str)) \indicatorf \setof{\eos \in \truncset\mleft(\palign\mleft(\cdot \mid \str \mright)\mright)} \prodstring \scaling(\palign(\tokent \mid \strltt)) \indicatorf \setof{\tokent \in \truncset\mleft(\palign\mleft(\cdot \mid \strltt \mright)\mright)} \\
   &= \palign(\eos \mid \str) \indicatorf \setof{\eos \in \truncset\mleft(\palign\mleft(\cdot \mid \str \mright)\mright)} \prodstring \palign(\tokent \mid \strltt) \indicatorf \setof{\tokent \in \truncset\mleft(\palign\mleft(\cdot \mid \strltt \mright) \mright)} \\
   &= \palign(\eos \mid \str) \indicatorf \setof{\eos \in \truncset\mleft(\palign\mleft(\cdot \mid \str \mright)\mright)} \prodstring \palign(\tokent \mid \strltt) \indicatorf \setof{\tokent \in \truncset\mleft(\palign\mleft(\cdot \mid \strltt \mright) \mright)} \\
   &= \palign(\eos \mid \str) \indicatorf \setof{\eos \in \truncset\mleft(\plm\mleft(\cdot \mid \str \mright)\mright)} \prodstring \palign(\tokent \mid \strltt) \indicatorf \setof{\tokent \in \truncset\mleft(\plm\mleft(\cdot \mid \strltt \mright) \mright)} \\
   &= \indicatorf \setof{\eos \in \truncset\mleft(\plm\mleft(\cdot \mid \str \mright)\mright)} \prodstring \indicatorf \setof{\tokent \in \truncset\mleft(\plm\mleft(\cdot \mid \strltt \mright) \mright)} \palign(\eos \mid \str) \prodstring \plm(\tokent \mid \strltt) \\
   &= \indicatorf \setof{\eos \in \truncset\mleft(\plm\mleft(\cdot \mid \str \mright)\mright)} \prodstring \indicatorf \setof{\tokent \in \truncset\mleft(\plm\mleft(\cdot \mid \strltt \mright) \mright)} \palign(\str) \\
   &\propto \underbrace{\indicatorf \setof{\eos \in \truncset\mleft(\plm\mleft(\cdot \mid \str \mright)\mright)} \prodstring \indicatorf \setof{\tokent \in \truncset\mleft(\plm\mleft(\cdot \mid \strltt \mright) \mright)}}_{\fprior} \plm(\str) \exp\mleft(\frac{1}{\beta} \reward\mleft(\str\mright)\mright) \\
   &= \fprior \plm(\str) \exp\mleft(\frac{1}{\beta} \reward\mleft(\str\mright)\mright) \\
\end{align}
\end{subequations}
We see that, in this simplified setting, the resulting probability of a string under the adapted and aligned model $\palignsampling$ equals the prior probability scaled by the reward function and by the truncation factors.
Since different sampling adaptors affect this relationship differently, this simplified example suggests that an appropriate choice of globally normalized sampling adaptor can be used to effectively select strings based on their probability under the prior.
For example, using the transform function from top-$\topk$ sampling will lead to the generation of corpora with higher average probability under the prior,\footnote{For $\topk < |\eosalphabet|$.} and thus, by the anti-correlation derived in our paper, with lower average reward.

\subsection{Examples of Sampling Adaptors}\label{app:sampling-adaptor-examples}
We note that these largely correspond to the examples in \citet{meister-etal-2023-efficacy}.

\begin{example}
We recover \defn{ancestral sampling} when $\scaling(\cdot) = \identity(\cdot)$ and $\truncsett = \eosalphabet$.
\end{example}
\begin{example}
We recover \defn{temperature sampling} when $\scaling(\plm(\tokent \mid \strltt)) \propto \plm(\tokent \mid \strltt)^{\frac{1}{\temp}}$ and $\truncsett = \eosalphabet$.
\end{example}
\begin{example}
We recover \defn{top-$\topk$} sampling \citep{fan-etal-2018-hierarchical} when $\scaling(\cdot) = \identity(\cdot)$ and
\begin{equation}
\truncsett = \argmax_{\eosalphabet' \subseteq \eosalphabet} \sum_{\token \in \eosalphabet'} \plm(\token \mid \strltt) \qquad s.t. \, |\eosalphabet'| = \topk
\end{equation}
i.e., the set of top-$\topk$ most probable symbols.
\end{example}
\begin{example}
We recover \defn{locally typical sampling} \citep{meister-etal-2023-locally} when when $\scaling(\cdot) = \identity(\cdot)$ and 
\begin{equation}
    \truncsett = \argmin_{\eosalphabet' \subseteq \eosalphabet} \sum_{\token \in \eosalphabet'}\big|\ent(\plm(\cdot \mid \strltt)) + \log \plm(\token \mid \strltt) \big| \qquad s.t. \, \sum_{\token \in \eosalphabet'} \plm(\token \mid \strltt) \geq \topp 
\end{equation}
i.e., the set of items with log-probability closest to the symbol-level entropy that collectively have probability mass $\geq \topp$.
\end{example}
\begin{example}
We recover \defn{$\eta$-sampling} \citep{hewitt-etal-2022-truncation} when $\scaling(\cdot) = \identity(\cdot)$ and 
\begin{equation}
    \truncsett = \setof{\token \in \eosalphabet \mid \plm(\token \mid \strltt) > \eta}
\end{equation}
that is, the set of symbols with probability greater than $\eta$, where $\eta = \min(\epsilon, \sqrt{\epsilon} \exp(-\ent(\plm(\cdot \mid \strltt)))))$. 
\end{example}

\newpage
\section{The Probability--Quality Trade-off in DPO-aligned Language Models}\label{sec:langmodelexperiments-dpo}
The probability--quality trade-off also applies to models aligned with direct preference optimization \citep[DPO;][]{rafailov2023direct}. 
Though an explicit reward function is not needed to train a language model with DPO, the training scheme maximises the same backward KL divergence objective as RLHF \citep[\Cref{eq:rlhf-objective};][]{korbak2023rlhf,rafailov2023direct,azar2023general}. 
We should thus expect that \Cref{prop:trade-off} applies to these models and observe the trade-off when we construct a reward function $\reward_\qalign$ using the prior $\plm$ and aligned model $\qalign$ as in \Cref{eq:secret}.
We employ the same setup as in \Cref{sec:langmodelexperiments}.

\paragraph{Models.}
We use the 7B DPO-aligned and prior language models from \citet{lee2024aligning}, both of which are based on Mistral 7B v0.2 \citep{jiang2023mistral}.
The DPO-aligned model is fine-tuned on the Multifaceted Collection \citep{lee2024aligning}, a dataset with 192k samples capturing preferences of style (e.g., clarity, tone), informativeness, and harmlessness, among others.
We construct the ``secret'' reward function as $\reward_\qalign(\str) = \frac{\qalign(\str)}{\plm(\str)}$, omitting the constant term.\looseness=-1

\paragraph{Results.}
We observe results identical to the setting with RLHF-tuned models.
Specifically, we observe a strong anti-correlation (Pearson correlation of $\pearson=-0.97$), trade-off control using sampling adaptors, and the emergence of Simpson's paradox. 
These are expected since RLHF and DPO have the same minimization objective, thus supporting our formal arguments in \Cref{sec:bayesian}. 

\begin{figure*}[ht!]
    \centering
    \includegraphics[width=\textwidth,keepaspectratio]{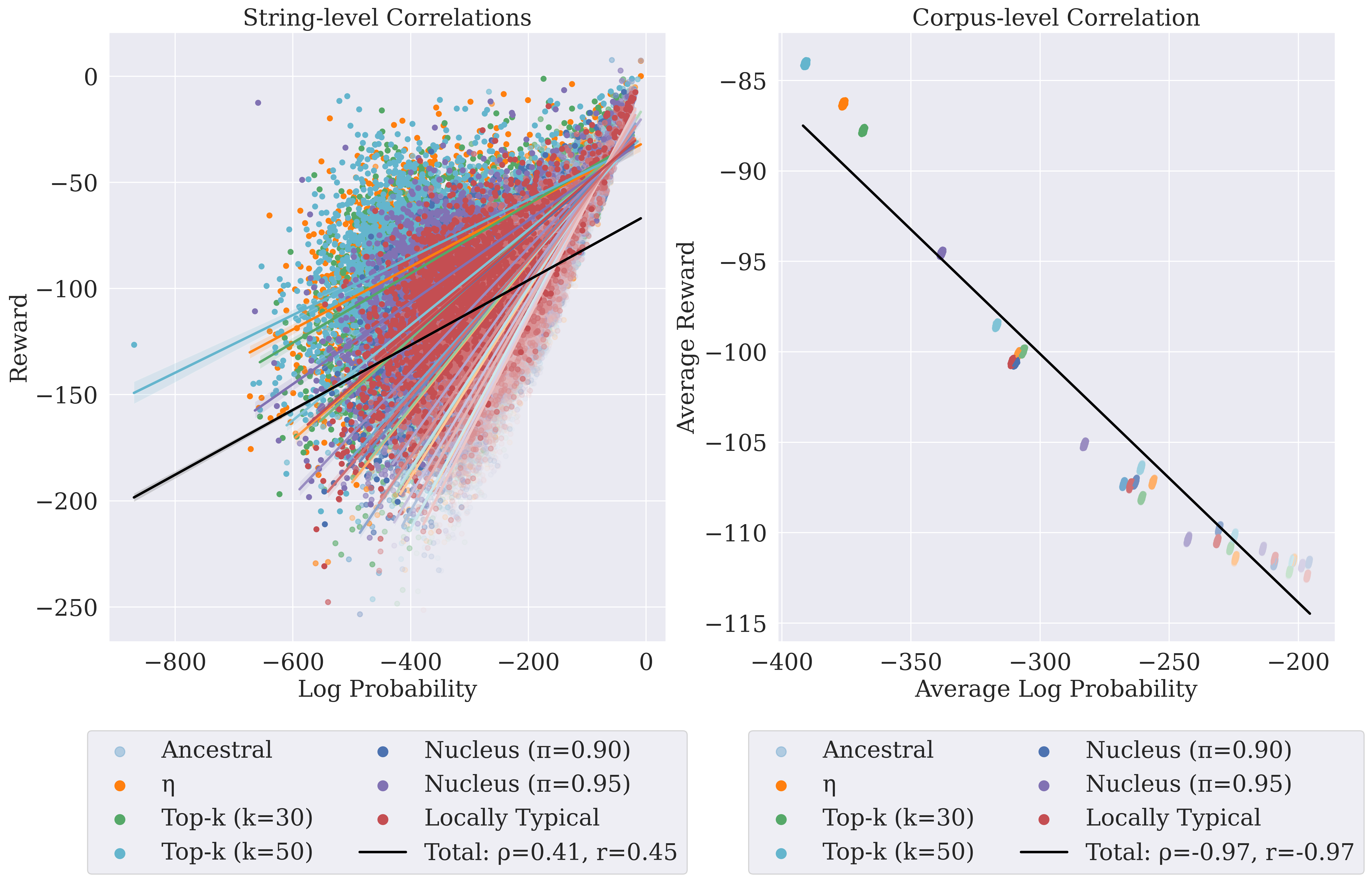}
    \caption{The probability--quality relationship in DPO-tuned models, where quality is measured by the secret reward function. (Left) String-level correlations between log-probability and quality. (Right) Corpus-level correlations between average log-probability and average quality, with corpora created by different sampling adaptors. Higher intensity of the colours denote higher temperatures used with the sampling adaptor.\looseness=-1}
    \label{fig:simpsons-dpo}
\end{figure*}

\newpage
\section{Toy Experiment Distributions}\label{fig:toy}
\begin{figure*}[h!]
    \centering
    \includegraphics[width=0.9 \textwidth,keepaspectratio]{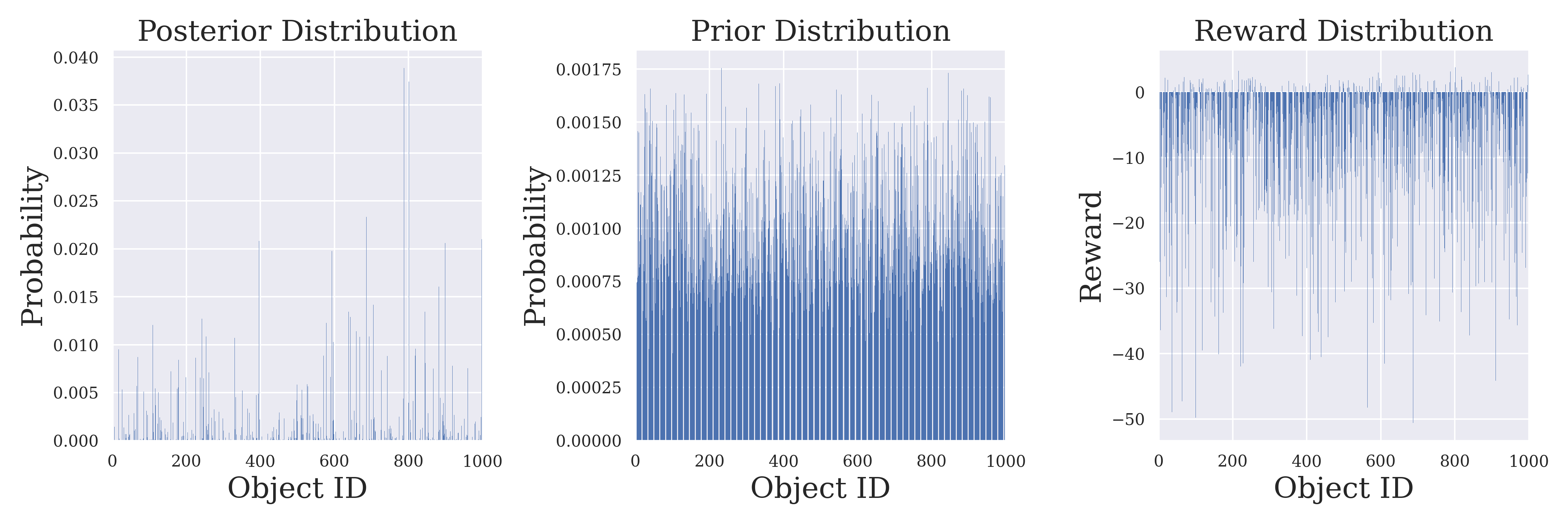}
    \caption{Toy models of $\palign(\datapoint)$, $\plm(\datapoint)$ and $\reward(\datapoint)$ analogous to the distributions over strings.}
   
\end{figure*}

\end{document}

%% file: macros.tex
\newcommand{\mymacro}[1]{{#1}}

\usepackage{cleveref}
\crefname{section}{\S}{\S\S}
\Crefname{section}{\S}{\S\S}
\crefformat{section}{\S#2#1#3}
\crefname{figure}{Fig.}{Fig.}
\crefname{alg}{Alg.}{Alg.}
\crefname{line}{line}{lines}
\crefname{appendix}{App.}{Apps.}
\crefname{equation}{eq.}{eqs.}
\crefname{table}{Tab.}{Tables}
\crefname{proposition}{Prop.}{Props.}
\crefname{principle}{Principle}{Principles}
\crefname{remark}{Remark}{Remarks}
\crefname{assumption}{Assumption}{Assumptions}
\crefname{hypothesis}{Hypothesis}{Hypotheses}

\crefformat{footnote}{#2\footnotemark[#1]#3}

\usepackage{thmtools}
\usepackage{thm-restate}
\usepackage{etoolbox}
\theoremstyle{plain} %
\declaretheorem[name=Theorem]{theorem}

\newtheorem{definition}{Definition}
\newtheorem{lemma}[theorem]{Lemma}
\newtheorem{proposition}{Proposition}
\newtheorem{corollary}{Corollary}
\newtheorem{example}{Example}

\newcommand{\defn}[1]{\mymacro{\textbf{#1}}}
\newcommand{\defeq}[0]{\mathrel{\mymacro{\stackrel{\textnormal{\tiny def}}{=}}}}

\newenvironment{proofsketch}{%
\proof}{\endproof}

\newcommand{\Renyi}{\mymacro{R{\'e}nyi}\xspace}
\newcommand{\renyient}{\mymacro{\ent_{\renyialpha}}}
\newcommand{\renyialpha}{\mymacro{\gamma}}
\newcommand{\gap}{\mymacro{\Delta}}
\newcommand{\gapalpha}{\mymacro{\gap\renyialpha}}

\newcommand{\reward}{{\mymacro{r}}}
\newcommand{\rewardbound}{{\mymacro{B}}}
\newcommand{\Aset}{\mymacro{\mathcal{A}}}
\newcommand{\good}{\mymacro{\textsc{+}}}
\newcommand{\bad}{\mymacro{\textsc{--}}}
\newcommand{\qalign}{\mymacro{q_{\good}}}
\newcommand{\palign}{\mymacro{\plm_{\good}}}
\newcommand{\normalizingconstant}{\mymacro{Z}}

\newcommand{\normalizingconstantrlhf}{\mymacro{\normalizingconstant(\good)}}

\newcommand{\anejchange}[1]{{#1}}
\newcommand{\naamanchange}[1]{{#1}}

\newcommand{\acceptancer}{\mymacro{a}}

\newcommand{\rvY}{\mymacro{\boldsymbol{Y}}}
\newcommand{\rvI}{\mymacro{I}}

\newcommand{\rvA}{\mymacro{A}}

\newcommand{\information}{\mymacro{\iota}}
\newcommand{\prob}{\mymacro{\mathbb{P}}}

\newcommand{\var}{\mymacro{\mathbb{V}}}
\newcommand{\typicalset}{ \mymacro{T_N^{\varepsilon}}}

\newcommand{\bigo}{\mymacro{\mathcal{O}}}

\newcommand{\Rpos}{\mymacro{\mathbbm{R}_{> 0}}}

\newcommand{\indicatorf}{\mymacro{\mathbbm{1}}}
\newcommand{\setof}[1]{\mymacro{\{ #1 \}}}
\newcommand{\powerset}{\mymacro{\mathcal{P}}}

\newcommand{\simplex}{\mymacro{\Delta}}
\newcommand{\ent}{\mymacro{\mathrm{H}}}
\newcommand{\identity}{\mymacro{\mathbb{I}}}

\DeclareMathOperator*{\argmax}{\mymacro{argmax}}
\DeclareMathOperator*{\argmin}{\mymacro{argmin}}

\newcommand{\KL}{\mymacro{\mathrm{KL}}}
\newcommand{\KLdiv}[2]{\mymacro{\KL( #1 \mid\mid #2)}}
\DeclareMathOperator*{\expected}{\mymacro{\mathbb{E}}}

\newcommand{\token}{\mymacro{y}}
\newcommand{\eostoken}{\mymacro{\overline{\token}}}
\newcommand{\tokent}{\mymacro{\token_t}}

\newcommand{\str}{\mymacro{\boldsymbol{y}}}
\newcommand{\strprime}{\mymacro{\str'}}
\newcommand{\strltt}{\mymacro{\str_{<t}}}
\newcommand{\strn}{\mymacro{\str^{(n)}}}

\newcommand{\strset}{\mymacro{\mathcal{Y}}}

\newcommand{\alphabet}{\mymacro{{\Sigma}}}
\newcommand{\eosalphabet}{\mymacro{\overline{\alphabet}}}
\newcommand{\kleene}[1]{\mymacro{{#1^*}}}

\newcommand{\sumoverstrings}{\mymacro{\sum_{\str \in \kleene{\alphabet}}}}
\newcommand{\sumoverstringsprime}{\mymacro{\sum_{\strprime \in \kleene{\alphabet}}}}
\newcommand{\sumovereostokens}{\mymacro{\sum_{\eostoken \in \eosalphabet}}}

\newcommand{\prodstring}{\mymacro{\prod^{|\str|}_{t=1}}}

\newcommand{\domain}{\mymacro{\mathcal{D}}}
\newcommand{\datapoint}{\mymacro{x}}

\newcommand{\fprior}{F_\globalsamplingadaptor\mleft[{\plm(\str)}\mright]}

\newcommand{\plm}{\mymacro{p}}
\newcommand{\eos}{\mymacro{\textsc{eos}}\xspace}

\newcommand{\plmlocalsampling}{\mymacro{\ddot{\plm}}}
\newcommand{\plmsampling}{\mymacro{\widetilde{\plm}}}
\newcommand{\palignsampling}{\mymacro{\plmsampling_+}}
\newcommand{\globalsamplingadaptor}{\mymacro{\boldsymbol{\gamma}}}
\newcommand{\normalizingconstantstring}{\mymacro{\normalizingconstant_\globalsamplingadaptor(\alphabet)}}
\newcommand{\truncset}{\mymacro{C}}
\newcommand{\truncsett}{\mymacro{\truncset\mleft(\plm(\cdot \mid \strltt)\mright)}}

\newcommand{\simplextoken}{\mymacro{\simplex^{\mid \eosalphabet \mid -1}}}
\newcommand{\Rpostoken}{\mymacro{\Rpos^{\mid \eosalphabet \mid}}}
\newcommand{\scaling}{\mymacro{\kappa}}

\newcommand{\topk}{\mymacro{k}}
\newcommand{\topp}{\mymacro{\pi}}
\newcommand{\temp}{\mymacro{\tau}}

\newcommand{\spearman}{\mymacro{\rho}}
\newcommand{\pearson}{\mymacro{r}}